\pdfoutput=1
\documentclass[11pt]{article}
\usepackage[textwidth=400pt,centering]{geometry}
\usepackage{fullpage}
\usepackage{microtype}
\usepackage{graphicx}
\usepackage{subfigure}
\usepackage{booktabs} 
\usepackage[sort]{natbib}
\usepackage[nottoc,notlot,notlof]{tocbibind}
\usepackage[vlined, linesnumbered, ruled]{algorithm2e}

\usepackage{amsmath, bm, amssymb, mathtools}
\usepackage{physics}

\def\eqref#1{equation~\ref{#1}}

\def\floor#1{\lfloor #1 \rfloor}
\def\1{\bm{1}}

\def\inp#1#2{\left\langle #1, #2 \right\rangle}

\DeclareMathAlphabet{\mathsfit}{\encodingdefault}{\sfdefault}{m}{sl}
\SetMathAlphabet{\mathsfit}{bold}{\encodingdefault}{\sfdefault}{bx}{n}

\def\gA{{\mathcal{A}}}
\def\gB{{\mathcal{B}}}

\def\gD{{\mathcal{D}}}
\def\gE{{\mathcal{E}}}

\def\gO{{\mathcal{O}}}
\def\gP{{\mathcal{P}}}

\def\gT{{\mathcal{T}}}

\def\sN{{\mathbb{N}}}

\def\sR{{\mathbb{R}}}

\newcommand{\E}{\mathbb{E}}

\newcommand{\Reg}{\mathrm{Reg}}

\newcommand{\I}{\mathbbm{1}}

\DeclareMathOperator*{\argmax}{arg\,max}
\DeclareMathOperator*{\argmin}{arg\,min}

\usepackage[hidelinks]{hyperref}
\hypersetup{
    colorlinks,
    linkcolor={red!50!black},
    citecolor={blue!50!black},
    urlcolor={blue!80!black}
}

\usepackage{csquotes}
\usepackage{amsthm, thmtools,thm-restate}
\usepackage[capitalize,noabbrev]{cleveref}
\theoremstyle{plain}
\newtheorem{theorem}{Theorem}

\newtheorem{lemma}[theorem]{Lemma}

\theoremstyle{definition}

\newtheorem{assumption}{Assumption}

\theoremstyle{remark}
\newtheorem{remark}[theorem]{Remark}

\usepackage[utf8]{inputenc} 
\usepackage[T1]{fontenc}    
\usepackage{url}            
\usepackage{nicefrac}       
\usepackage{xcolor}    

\usepackage{pifont}
\usepackage{bbm}
\usepackage{multirow}
\usepackage{braket}
\usepackage{enumitem}
\definecolor{darkblue}{rgb}{0,0.08,0.8}
\newcommand\numberthis{\addtocounter{equation}{1}\tag{\theequation}}

\usepackage{verbatim}

\newcommand{\ul}{\underline{\lambda}}
\newcommand{\uL}{\underline{L}}

\usepackage{calc}  
\usepackage{array} 

\title{Adaptive Learning Rates with Surrogate Probability \\ for Follow-the-Perturbed-Leader}
\author{Jongyeong Lee$^{1}$ \\ \fontsize{10}{12}\selectfont jongyeong@kist.re.kr \and Junya Honda$^{2,3}$ \\ \fontsize{10}{12}\selectfont honda@i.kyoto-u.ac.jp \and Shinji Ito$^{4,3}$ \\ \fontsize{10}{12}\selectfont shinji@mist.i.u-tokyo.ac.jp \and Chansoo Kim$^{1,5}$ \\ \fontsize{10}{12}\selectfont eau@ust.ac.kr}
\date{\fontsize{10}{12}\selectfont
$^1$ Korea Institute of Science and Technology 
$^2$ Kyoto University
$^3$ RIKEN AIP \\
$^4$ The University of Tokyo
$^5$ University of Science and Technology
}
\usepackage{times}

\begin{document}
\maketitle

\begin{abstract}%
    Follow-the-regularized-leader framework has shown effectiveness and flexibility in online learning problems, where the choice of learning rates are known to be crucial.
    Recently, adaptive learning rates defined in terms of the arm-selection probabilities, obtained by solving convex optimization, have achieved improved best-of-both-worlds (BOBW) guarantees in various bandit problems.
    In contrast, BOBW guarantees for its computationally efficient alternative, follow-the-perturbed-leader (FTPL), remain relatively limited since its optimization-free nature ironically makes the design of adaptive, probability-dependent learning rates non-trivial.
    To address this challenge, we propose an adaptive learning rate for FTPL by introducing surrogate probability functions that can be computed only from the available quantities, without requiring the exact probabilities.
    Based on these learning rates with surrogate functions, we provide the BOBW guarantee for FTPL with Pareto perturbations for any shape parameter $\alpha >1$, generalizing prior results restricted to specific choices of $\alpha=2$.
    We further show the BOBW guarantees for FTPL with adaptive learning rates in the bandit problem with expert advices.
    Our approach preserves the computational simplicity of FTPL while enabling probability-dependent adaptivity, and the surrogate-based methodology may be of independent interest in other algorithmic frameworks beyond FTPL and learning rate designs.
\end{abstract}

\section{Introduction}
The multi-armed bandit (MAB) problem is a fundamental problem for sequential decision making under uncertainty.
In this problem, at each round $t \in [T] :=\qty{1,\ldots, T}$, an agent selects an arm $i_t$ from a set of $K$ arms and observes the corresponding loss $\ell_{t,i_t}$ from $i_t$, where the loss vectors $\ell_t=(\ell_{t,1},\ldots, \ell_{t,K})\in [0,1]^K$ are determined by the environment.
A central challenge in bandit problems is to minimize the cumulative loss by learning the environments only with partial feedback, where two canonical types of environments have been extensively studied.

In the stochastic regime, losses are identically independently distributed (i.i.d.) from an unknown but fixed distribution~\citep{lai1985asymptotically, katehakis1995sequential}.
In contrast, in the adversarial regime, losses may be chosen by an adversary, possibly adaptively based on the agent's past actions, so that no distributional assumption can be made~\citep{auer2002nonstochastic}.
Since the true nature of the environment is usually unknown in practice, there has been considerable interest in policies that achieve (near-)optimal performance guarantees in both regimes, which is referred to as Best-of-Both-Worlds (BOBW) guarantee~\citep{bubeck2012best, seldin2017improved}.

A dominant approach to achieving BOBW guarantees would be the Follow-the-Regularized-Leader (FTRL) framework, which has successfully obtained BOBW guarantees across a wide range of online learning problems including the standard multi-armed bandits~\citep{zimmert2021tsallis, jin2023improved}, decoupled bandits~\citep{rouyer2020decoupled}, combinatorial-semi bandit~\citep{zimmert2019beating,ito21a}, and partial monitoring~\citep{tsuchiya2024simple}, to name a few.
This success can be attributed to its generality and flexibility: one can tailor the policy to the problem structure by carefully choosing the (convex) regularizers and learning rates.
However, this generality comes at a cost, as it requires solving a convex optimization problem at each round to compute the arm-selection distribution, which can be computationally demanding in practice.

This limitation has motivated a line of work aimed at avoiding per-round optimization while preserving desirable regret guarantees.
One approach is to approximate FTRL updates using simpler arithmetic operations, as in the Prod family of policies~\citep{cesa2007improved, gaillard2014second}, where \citet{zimmert2024productive} obtained BOBW guarantee by leveraging a kind of first order approximation of FTRL with Tsallis entropy.
Another prominent alternative is the Follow-the-Perturbed-Leader (FTPL) framework, which selects arms by adding random perturbations to cumulative losses~\citep{poland2005fpl,kalai2005ftpl}.
FTPL has gained attention due to its simplicity and (almost) optimization-free nature.
Moreover, there is a deep theoretical correspondence between FTRL and FTPL, where specific FTRL is associated with particular perturbation distributions in FTPL~\citep{abernethy2016perturbation, li2024optimism, lee2025}.
Recent work has also established BOBW guarantees for FTPL in standard MABs~\citep{pmlr-v201-honda23a, lee2024follow}, decoupled bandits~\citep{kim2025follow}, and combinatorial semi-bandits~\citep{zhan2025followtheperturbedleader,chen2026further}.

Despite these recent advances, progress on BOBW guarantees for FTPL still remains limited compared with that for FTRL.
In particular, recent advances in the FTRL framework exploit explicit arm-selection probabilities to design adaptive learning rates and employ hybrid regularizers, to obtain BOBW guarantees in various settings~\citep{jin2023improved,tsuchiya2024simple,zhao2025heavy}.
However, these techniques do not transfer straightforwardly to FTPL, where probabilities do not admit a closed form and are induced implicitly only through perturbations.
In other words, while perturbations remove the need for optimization, they conversely appear to degrade the adaptivity of FTPL that has been crucial to recent advances in FTRL framework with adaptive learning rates.

\paragraph{Contribution.}
In this paper, we address this gap by developing adaptive learning rates for FTPL that preserve computational efficiency while achieving BOBW guarantees, \emph{without requiring explicit arm-selection probabilities}.
Our approach is inspired by the stability-penalty matching (SPM)  methodology developed for FTRL frameworks~\citep{pmlr-v247-ito24a, nguyen2025data} and by recent uses of surrogate probabilities in FTPL~\citep{kim2025follow}.
The key technical idea is to introduce surrogate probability functions that replace true arm-selection probabilities and can be computed solely from the currently available quantities.

Based on SPM learning rates with surrogate probabilities, we first generalize existing BOBW guarantee of FTPL in the standard MAB.
In particular, we show that FTPL with Pareto perturbations of any shape $\alpha >1$ achieves BOBW guarantee, whereas previous results were restricted to the case $\alpha=2$~\citep{lee2024follow}.
This result aligns with those obtained for FTRL with $\gamma$-Tsallis entropy for general $\gamma \in (0,1)$~\citep{jin2023improved, pmlr-v247-ito24a}, extending beyond the case $\gamma=1/2$~\citep{zimmert2021tsallis}.
We further extend the BOBW guarantee for FTPL with SPM learning rates to the bandit problems with expert advice~\citep{auer2002nonstochastic}, which is also referred to as contextual bandits~\citep{dann2023blackbox}.

One of the main advantage of our approaches over FTRL-based approaches lies in its simplicity and computational efficiency, especially when the hybrid regularizers are considered.
While FTRL only with Tsallis entropy admits a solution that can be formulated as one-dimensional optimization problems, which can be efficiently computed via Newton’s method or bisection method~\citep{zimmert2021tsallis}, recent improvements in regret guarantees often rely on hybrid regularizers~\citep{tsuchiya2023further, jin2023improved, pmlr-v247-ito24a, nguyen2025data}.
These regularizers may require solving a convex optimization problem at every round.
In contrast, our policy follows the standard FTPL framework with Pareto perturbations and computes learning rates directly from currently available quantities, completely avoiding convex optimizations.

Besides its computational efficiency, the use of surrogate probability functions that do not belong to the probability simplex may be of independent interest.
Although our analysis is grounded in FTPL, this surrogate-based approach could also be applicable to other algorithmic frameworks beyond FTPL.
More broadly, it could extend to settings where explicit probability vectors are used, not limited to the design of adaptive learning rates.
For example, in heavy-tailed bandits, \citet{huang2022adaptive} addressed the effect of extreme observations by skipping large losses, where the skipping thresholds are chosen adaptively based on the arm-selection probability.
We expect that our approach could be used to replace such explicit arm-selection probabilities with surrogate probability functions.

\section{Preliminaries}
In this section, we introduce notation and formulate the problem.
Then, we introduce the intuition behind stability-penalty matching (SPM) methods in FTRL frameworks.

\subsection{Problem formulation}
In bandit settings, the environment determines the loss vector $\ell_t \in [0,1]^K$ and the agent selects an arm $i_t$ at each round, where the performance of the agent's policy is measured by the pseudo-regret.
When $w_t$ denotes the arm-selection probabilities of policy at round $t$, the pseudo-regret is defined by
\begin{equation*}
    \Reg(T) = \E\qty[\sum_{t=1}^T \inp{\ell_t}{w_t - e_{i^*}}], \, i^* \in \argmin_{i \in [K]} \E\qty[\sum_{t=1}^T \ell_{t,i}].
\end{equation*}
Here, $i^*$ denotes the optimal arm in hindsight and is assumed to be unique following the prior studies~\citep{lee2024follow, pmlr-v247-ito24a}.
Since only partial feedback is observable, an estimator $\hat{\ell}_t$ of the loss vector $\ell_t$ can be used, which is specified later.

In this paper, we considers two possible environments, the adversarial regime~\citep{auer2002nonstochastic} and adversarial regime with self-bounding constraints~\citep{zimmert2021tsallis}.
In the adversarial regime, $\ell_t$ is determined in an adversarial way possibly depending on the history, $\{(\ell_s, i_s)\}_{s=1}^{t-1}$.
The adversarial regime with a $(\Delta, C, T)$ self-bounding constraint is an environment where the regret can be bounded from below as follows:
\begin{equation*}
    \Reg(T) \geq \Reg'(T) - C, \text{ where } \Reg'(T) = \E\qty[\sum_{t=1}^T \Delta_{i_t}] = \E\qty[\sum_{t=1}^T \sum_{i=1}^K \Delta_{i}w_{t,i}]. 
\end{equation*}
This regime also includes the stochastic environments with adversarial corruption~\citep{wei2018adversarial}, where each $\Delta_i \geq 0$ is equivalent to the suboptimality gap of arm $i$ and $C$ is the magnitude of corruption.
Note that the unique optimal $i^*$ assumption implies $\Delta_{i}>0$ holds for all $i\ne i^*$, where we denote $\Delta_{\min} = \min_{i\ne i^*} \Delta_i$ in the adversarial regime with $(\Delta,C,T)$ self-bounding constraint.

\subsection{Follow-the-Perturbed-Leader}
Let $\hat{L}_t = \sum_{s=1}^{t-1}\hat{\ell}_{s}$ be the cumulative loss estimator up to round $t-1$.
Then, FTPL is a policy that selects an arm $i_t$ according to
\begin{equation}\label{def: ftpl it}
    i_t = \argmin_{i\in [K]} \qty{\hat{L}_{t,i} - \frac{r_{t,i}}{\eta_{t}}} = \argmin_{i\in [K]} \qty{\eta_t \hat{\uL}_{t,i} - r_{t,i}}, \text{ where } r_{t,i} \stackrel{\text{i.i.d.}}{\sim} \gD, \, \forall i \in [K].
\end{equation}
Here, the underline denotes the gap of a vector from its minimum, i.e., $\ul=\lambda - \1 \min_{i\in [K]} \lambda_i$ for all-one vector $\1$, and the learning rate $\eta_t$ will be defined later.
In this paper, we consider perturbations generated from the shifted Pareto distribution with shape $\alpha$, also known as the Lomax distribution with shape $\alpha$ and scale $1$, whose density function $f$ and distribution function $F$ are defined by
\begin{equation}\label{def: f and F of Lomax}
    f(z) = \frac{\alpha}{(z+1)^{\alpha+1}}, \text{ and } F(z) = 1- \frac{1}{(1+z)^\alpha}, \, \forall z \geq 0.
\end{equation}
Then, the arm-selection probability given $\hat{L}_t$ can be written as $w_{t,i} = \mathrm{Pr}[i_t =i | \hat{L}_t] = \phi_i(\eta_t \hat{L}_t)$, where for $\lambda \in [0,\infty)^K$
\begin{align}\label{def: ftpl wt}
    \phi_i(\lambda) := \Pr_{r_1, \ldots, r_K \sim \gD}\qty[i_t = i\middle | \lambda] = \int_{0}^\infty f(z+\ul_i) \prod_{j\ne i} F(z+\ul_j) \dd z.
\end{align}
While the importance-weighted (IW) estimator $\hat{\ell}_{t,i} =\I[i_t=i]\ell_{t,i}/w_{t,i}$ is commonly used, $w_t$ of FTPL in (\ref{def: ftpl wt}) generally does not admit a closed form, which complicates its direct use.
Therefore, it is standard to estimate $1/w_{t,i}$ in FTPL via resampling-based procedures~\citep{abernethy2016perturbation, pmlr-v201-honda23a}.
In particular, \citet{neu2016importance} proposed the geometric resampling (GR) method, which produces an unbiased estimator of $1/w_{t,i_t}$ by repeatedly sampling perturbations $r_t'$ from the perturbation distribution until the FTPL rule selects the same $i_t$ arm with resampled $r_t'$.

\subsection{Stability-penalty matching learning rates for FTRL}\label{sec: FTRL spm}
In the standard regret decomposition of FTRL, it is well known that the regret is bounded from above as~\citep[Exercise 28.12]{lattimore2020bandit}
\begin{align}\label{eq: Reg FTRL}
    \Reg(T) \lesssim \underbrace{\sum_{t=1}^T \eta_t z_t}_{\text{stability term}} +\underbrace{ \sum_{t=1}^T \qty(\frac{1}{\eta_{t+1}} - \frac{1}{\eta_{t}})h_{t+1}}_{\text{penalty term}},
\end{align}
where the formulations of $z_t$ and $h_t$ depend on the problem setting and the regularizer function of FTRL.
While the learning rates $\eta_t$ may be designed as a function of either $z_t$ or $h_t$, recent advances in the analysis of FTRL suggest using $\eta_t$ so that the contributions of the stability and penalty terms to be of the same order, which is referred to as stability–penalty matching (SPM)~\citep{jin2023improved,tsuchiya2023stability}.
In particular, for $\beta_t := \eta_t^{-1}$, the update of SPM learning rates roughly takes the form of
\begin{equation}\label{eq: SPM for FTRL}
  \beta_{t+1} = \beta_t + \frac{z_t}{\beta_t h_{t+1}} \implies \Reg(T) \lesssim \sum_{t=1}^T \frac{z_t}{\beta_t},
\end{equation}
where the appropriate choices of $z_t$ and $h_t$ have shown the BOBW guarantees in various bandit problems~\citep{pmlr-v247-ito24a, zhao2025heavy}.
While the idea of SPM is intuitive, a technical obstacle arises from the appearance of $h_{t+1}$ in the update rule since computing $h_{t+1}$ usually requires information from the next round, which itself depends on $\beta_{t+1}$.
To resolve this circular dependence, prior FTRL policies employed hybrid regularizers to ensure that $h_{t+1}=\gO(h_t)$, thereby justifying the use of $h_t$ in place of $h_{t+1}$ in the update of $\beta_{t+1}$ in (\ref{eq: SPM for FTRL})~\citep{pmlr-v247-ito24a, nguyen2025data}.

In the standard FTRL analysis, $z_t$ usually depends on the arm-selection probabilities $w_{t,i}$ and $h_t$ is determined by the value of the regularizer at $w_{t,i}$.
For example, in the multi-armed bandits, FTRL with $\gamma$-Tsallis entropy satisfies $z_t \lesssim \sum_i (\E[w_{t,i}])^{1-\gamma}$ and $h_t \approx \sum_i(\E[w_{t,i}])^{\gamma}$~\citep{zimmert2021tsallis}.
Hence, the update of $\beta_{t+1}$ in \citet{pmlr-v247-ito24a} explicitly relies on the values of $w_{t,i}$ at each round, which are obtained by solving the associated convex optimization in FTRL.

\section{Adaptive learning rates for FTPL in multi-armed bandits}\label{sec: mab}
In this section, we propose an adaptive learning rate under SPM principle tailored to FTPL, which can be computed only from quantities available at the current round.

Given $\hat{L}_t$, we define $\sigma_{t,i}$ as the rank of $\hat{L}_{t,i}$ among $\{\hat{L}_{t,j}\}_{j \in [K]}$, where $\sigma_{t,i}=1$ if $\hat{L}_{t,i}$ is the smallest and $\sigma_{t,i}=K$ to the largest with arbitrary tie-breaking rule.
Let $j_t$ be the arm satisfying $\sigma_{t,j_t}=1$ after tie-breaking.
Note that even if multiple arms satisfy $\hat{\uL}_{t,i} = 0$, the tie-breaking rule ensures that there is a unique arm with $\sigma_{t,i}=1$.

\subsection{FTPL with conditional geometric resampling}\label{sec: FTPL CGR}
Since $w_t$ of FTPL does not admit the closed form in general, constructing the IW loss estimator $\hat{\ell}_t$ requires estimating $1/w_{t,i_t}$.
In the standard multi-armed bandit setting, we adopt conditional geometric resampling (CGR)~\citep{chen2025geometric}, which improves the computational efficiency of the original GR~\citep{neu2016importance}.
Beyond its computational advantages, CGR (precisely, CGR II-biased) provides a bounded loss estimator without sacrificing BOBW guarantees, which also simplifies the analysis, especially for $\alpha \in (1,2)$.
For $\alpha \geq 2$, the use of CGR II-biased is not essential as it is also possible to obtain BOBW guarantee with the original GR by applying the results in previous BOBW analysis~\citep{pmlr-v201-honda23a,lee2024follow}.

The main idea of CGR is to avoid generating new perturbations that clearly violate the termination condition, under which the FTPL rule in (\ref{def: ftpl it}) never selects $i_t$.
This is achieved by restricting the perturbation distributions during the resampling step to satisfy certain necessary conditions.
In particular, we employ the CGR II that generates the fresh perturbations from appropriately truncated conditional distributions so that resampled perturbation for the selected arm $r_{t,i_t}'$ is larger than $\eta_t \hat{\uL}_{t,i_t}$ as well as $r_{t,j}'$ for all arms with $\sigma_{t,j} < \sigma_{t,i_t}$.
More details including explicit sampling distributions and explicit policy is provided in Appendix~\ref{app: CGR}.

Let $M_t$ denote the number of resampling steps until the same arm $i_t$ is selected again by FTPL rule with these resampled perturbations and $G_t$ denote the maximum number of allowed resampling steps.
Then, it was shown that $M_t \sigma_{t,i_t}/(1-F^{\sigma_{t,i_t}}(\eta_t \hat{\uL}_{t,i_t}))$ is an unbiased estimator of $w_{t,i_t}^{-1}$ when $G_t = \infty$~\citep{chen2025geometric}.
When $G_t$ is finite, this early stopping introduces bias and thus incurs additional term in the regret.
However, with an appropriate choice of $G_t$, such additional regret term is at most $\log T$.
In Appendix~\ref{app: CGR bias}, we show that $G_t = K\log t$ is sufficient for Pareto perturbations.
Moreover, our analysis can be extended to general Fr\'echet-type distributions, in contrast to the analysis in~\citet[Lemma 8]{chen2025geometric}, which relies on properties of the Fr\'echet distribution.

For $\alpha \in (1,2)$, we additionally introduce an event $\gE_{t,\alpha}$, where we further reduce the resampling budget $G_t$ to avoid redundant resampling, defined by
\begin{equation}\label{def: event for decomposition for small alpha}
    \gE_{t,\alpha} := \bigg\{\sum_{i \ne j_t} \frac{1}{(1+\eta_t\hat{\uL}_{t,i})^\alpha} < \frac{1}{2}\bigg\},\, \forall \alpha \in (1,2).
\end{equation}
Simply speaking, this event roughly corresponds to the case where the probability of selecting the current best arm $j_t$ is at least $1/2$.
Hence, when both $i_t = j_t$ and $\gE_{t,\alpha}$ occur, we reduce the resampling budget $G_t$ from $K\log t$ to $2\log t$ in order to prevent redundant resampling that can induce artificially large loss estimates.
This event also plays an important role in the regret decomposition including $\alpha \geq 2$, where the definition of $\gE_{t,\alpha}$ for $\alpha \geq 2$ involves an $\alpha$-dependent threshold that is slightly larger than $1/2$. 
The pseudo-code of overall policy is given in Algorithm~\ref{alg: FTPL MAB}.

\begin{algorithm}[t]
   \caption{FTPL with conditional geometric resampling II-biased and SPM learning rates}
   \label{alg: FTPL MAB}
   \SetAlgoVlined
   \DontPrintSemicolon
   \SetKwInOut{Input}{Input}
   \Input{$K \in \sN$, $\alpha >1$, $\beta_1>0$, $\hat{L}_1 = 0$.}
   \For{$t=1, 2,\ldots$}{
        Sample $r_t = (r_{t,1}, \ldots, r_{t,K})$ i.i.d. from the Pareto distribution with shape $\alpha$ in (\ref{def: f and F of Lomax}).
        
        Select $i_t \in \argmin_{i\in[K]} \{\hat{L}_{t,i} - \beta_tr_{t,i}\}$ and observe $\ell_{t,i_t}$. \tcp*{FTPL}

        Find $j_t \in \argmin_i \hat{L}_{t,i}$ and set $M_{t}:=0$ and $G_t := K \log t$.

        \textbf{if} $i_t = j_t$, $\alpha \in (1,2)$, and $\I[\gE_{t,\alpha}]=1$ in (\ref{def: event for decomposition for small alpha}) \textbf{then} $G_t := 2\log t$.

        \Repeat{$i_t = \argmin_{i\in[K]} \{\hat{L}_{t,i} - \beta_t r_{t,i}'\}$ \textnormal{or} $M_t \geq G_t$.}{
            $M_t := M_t +1$. \tcp*{CGR II-biased}

            Sample $r_t'$ from the appropriately defined conditional distribution, details in Appendix~\ref{app: CGR}.  
            


        }
        Set $\hat{\ell}_{t,i_t} =M_t \sigma_{t,i_t}\ell_{t,i_t}/(1-F^{\sigma_{t,i_t}}(\eta_t\hat{\uL}_{t,i_t}))$ and update $\hat{L}_{t+1} = \hat{L}_t + \hat{\ell}_{t,i_t}e_{i_t}$. 
        
        Set $\beta_{t+1}$ by the update rule of (\ref{def: lr in MAB}) based on $z_t, h_t$ in (\ref{def: hz in MAB}) and $q_t$ in (\ref{def: surrogate}).
   }
\end{algorithm}
\subsection{The surrogate probability function}
As discussed in Section~\ref{sec: FTRL spm}, SPM learning rates are designed to balance the stability and penalty terms, which are expressed explicitly as functions of the arm-selection probabilities $w_t$ in the FTRL framework~\citep{zimmert2021tsallis}.
Accordingly, the update for SPM learning rates $\beta_{t+1}$ in FTRL naturally relies on $w_t$~\citep{jin2023improved,pmlr-v247-ito24a}.
In contrast, since $w_t$ of FTPL does not admit a closed form, existing BOBW analyses for FTPL instead express the stability and penalty terms by the cumulative loss estimators $\hat{L}_t$ and learning rates $\eta_t$~\citep{pmlr-v201-honda23a, lee2024follow}.
This motivates replacing $w_t$ with suitable surrogate quantities that appear in the corresponding regret bounds.
In this paper, we consider the following two surrogate functions, defined for any $i\in [K]$ and $t\in \sN$ by
\begin{equation}\label{def: surrogate}
    p_{t,i} = \min \qty(\frac{1}{(1+\eta_t \hat{\uL}_{t,i})^{\alpha}}, \frac{1}{\sigma_{t,i}}), \text{ and }
    q_{t,i}= \min \qty(\frac{1}{(1+\eta_{t-1} \hat{\uL}_{t,i})^{\alpha}}, \frac{1}{\sigma_{t,i}}).
\end{equation}
When $\eta_t$ is non-increasing, $q_{t,i}\leq p_{t,i}$ always hold by definition, and one can easily show $w_{t,i}\leq p_{t,i}$, whose proof is given in Appendix~\ref{app: ineq} for completeness.
Therefore, $p_{t,i}$ can be seen as a surrogate probability function.
In particular, Lemma~\ref{lem: stab and pen} in Appendix~\ref{app: pt discuss} shows that the stability and penalty terms can be bounded from above in terms of $p_{t,i}$.
This observation naturally motivates the use of $p_{t,i}$ instead of $w_{t,i}$, where a similar approach was explored in decoupled bandits~\citep{kim2025follow}.

While this approach is analytically reasonable, directly using $p_t$ to design SPM learning rates as in (\ref{eq: SPM for FTRL}) unfortunately leads to technical difficulties.
As discussed in Section~\ref{sec: FTRL spm}, updating $\beta_{t+1}$ requires the value of the penalty term at round $t+1$, whose definition depends on $p_{t+1,i}$ and therefore implicitly on $\beta_{t+1}$ itself.
Although this circular dependency can, in principle, be resolved by iterative computation, doing so would introduce unpredictable computational overhead, which~loses motivation to use FTPL.
In the FTRL framework, circular dependency is resolved by employing hybrid regularizers, which leverages properties of the Bregman divergence and the exact probability $w_t$.
In contrast, since the surrogate $p_{t,i}$ directly depends on the cumulative loss $\hat{L}_t$, it is not straightforward to derive a clean relation between $p_{t+1,i}$ and $p_{t,i}$.
To circumvent these difficulties, we instead work with the surrogate $q_t$, which relies on the previous learning rate.
While a detailed discussion of the role of $p_t$ and its limitations is given in Appendix~\ref{app: pt discuss}, the following statement supports the use of $q_{t}$ instead of $p_t$. 
\begin{lemma}[Informal]\label{lem: pt and qt}
    For all $i\in [K]$, $q_{t,i}\leq p_{t,i}\leq 2q_{t,i}$ for certain learning rates $\eta_t$. 
\end{lemma}
Hence, in this paper, $q_t$ will define the actual policy while $p_t$ still serves as an analytical tool.

\subsection{Learning rates with the surrogate probability for stability-penalty matching}
The expected regret of FTPL with CGR II-biased can be decomposed into two terms, one for the regret by the policy itself and the other from the bias of the estimator~\citep[Lemma 7]{chen2025geometric} 
\begin{align*}
    \Reg(T) \leq \sum_{t=1}^T \E\qty[\inp{\hat{\ell}_t}{w_t - e_{i^*}}] + \Reg_{\text{CGR}}(T) := \Reg_{\text{FTPL}}(T)+ \Reg_{\text{CGR}}(T),
\end{align*}
where Appendix~\ref{app: CGR bias} provides the explicit form of $\Reg_{\text{CGR}}(T)$ and shows that it is bounded by $\log T$.
The first term is the main regret term, which can be decomposed as~\citep[Lemma~4]{kim2025follow}: 
\begin{align}\label{eq: Reg FTPL}
   \Reg_{\text{FTPL}}(T) \lesssim \underbrace{\sum_{t=1}^T \E\qty[ \inp{\hat{\ell}_t}{\phi(\eta_t \hat{L}_t) -\phi(\eta_t \hat{L}_{t+1}) }]}_{\text{stability term}} +\underbrace{ \sum_{t=1}^T \qty(\beta_{t+1} - \beta_t)\E\qty[r_{t+1,i_{t+1}}- r_{t+1,i^*}]}_{\text{penalty term}}. 
\end{align}
To design SPM learning rates following the idea as in (\ref{eq: SPM for FTRL}), it is crucial to identify appropriate $z_t$ and $h_t$, which are closely related to the bounds on the stability and penalty terms, respectively.
To this end, we define $z_t$ and $h_t$ by
\begin{equation}\label{def: hz in MAB}
    z_t = \alpha \sum_{i\ne j_{t+1}} q_{t+1,i}^{1/\alpha}, \text{ and } h_t = \frac{\alpha}{\alpha-1}\sum_{i\ne j_{t+1}} q_{t+1,i}^{1-\frac{1}{\alpha}}. 
\end{equation}
Recall that, by construction, $q_{t+1}$ depends only on the \emph{previous} learning rate $\eta_t$ and $\hat{L}_{t+1} =\hat{L}_t + \hat{\ell}_t$, both of which are available at the end of round $t$.
Hence, we index $z_t$ and $h_t$ by $t$ to align with the notation used for FTRL as in (\ref{eq: SPM for FTRL}).
To clarify the motivation behind these definitions, we compare them with the case of FTRL with $\gamma$-Tsallis entropy in multi-armed bandits.
In this setting, \citet{pmlr-v247-ito24a} defined
\begin{equation}\label{def: hz of FTRL}
  z_t^{\text{FTRL}} = \frac{1}{1-\gamma} \sum_{i=1}^K \tilde{w}_{t,i}^{1-\gamma} ,\text{ and }h_t^{\text{FTRL}} = \frac{1}{\gamma} \qty(\sum_{i=1}^K w_{t,i}^{\gamma}-1),
\end{equation}
where $\tilde{w}_{t,i} =\min(w_{t,i}, 1-w_{t,j_t})$ and $j_t$ denotes the current best arm at round $t$.
These quantities are constructed so that the regret bound can be expressed as in (\ref{eq: Reg FTRL}), with $z_t^{\text{FTRL}}$ and $h_{t+1}^{\text{FTRL}}$.
Moreover, \citet{pmlr-v247-ito24a} introduced hybrid regularizers to ensure that the arm-selection probability changes smoothly, i.e., $w_{t+1,i}=\gO(w_{t,i})$.
Therefore, the correspondence between $\gamma$-Tsallis entropy and Fr\'echet-type perturbation with shape $1/(1-\gamma)$ suggests that the analogous quantities would be of the formulation $\alpha w_{t,i}^{1/\alpha}$ or $\alpha w_{t+1,i}^{1/\alpha}$ when defining the FTPL counterpart of $z_t^{\text{FTRL}}$.
Here, Lemma~\ref{lem: pt and qt} implicitly reveals that our constructions of $z_t$ and $h_t$ in (\ref{def: hz in MAB}) satisfy the latter expression, $w_{t+1,i}$, through the dependence of $q_{t+1}$ on $\hat{L}_{t+1}$ and its relation to $p_{t+1}$.

With these definitions in (\ref{def: surrogate}) and (\ref{def: hz in MAB}), the FTPL regret in (\ref{eq: Reg FTPL}) can be roughly rewritten for any learning rate $\beta_t$ as
\begin{equation*}
    \Reg_\text{FTPL}(T) \lesssim \gO\qty(\sum_{t=1}^T \frac{z_t}{\beta_{t}} + (\beta_{t+1}-\beta_t) h_t  ),
\end{equation*}
which is analogous to those of FTRL in (\ref{eq: Reg FTRL}), with the main difference on \underline{the appearance of $h_t$} instead of $h_{t+1}$.
This shift is due to the use of the surrogate $q_{t+1,i}$ in $h_t$, which depends on the previous learning rate $\eta_{t}$ and ensures both $z_t$ and $h_t$ are measurable at the end of round $t$, thereby avoiding issues realted to $h_{t+1}$.
Based on this decomposition, we design the following adaptive learning rates to equalize the order of the stability and penalty terms, consistent with the idea of SPM:
\begin{equation}\label{def: lr in MAB} 
        \beta_{t+1} = \begin{cases}
        \min(2^{\frac{1}{\alpha}}\beta_t, \beta_t + \frac{z_t}{\beta_t h_t}) , & \text{if } \alpha \geq 2,\\
        \beta_t + \max\qty(\frac{z_t}{\beta_t h_t}, \frac{4}{2^{1/\alpha}-1}\frac{1}{t}) , & \text{if } \alpha \in (1,2),
    \end{cases}
    \text{ and } \beta_1 \geq 2\alpha K^{\frac{1}{2}-\frac{1}{\alpha}}.\\
\end{equation}
While FTRL approaches employed additional regularizer such as log-barrier~\citep{jin2023improved} or complement Tsallis entropy~\citep{pmlr-v247-ito24a} based on the parameter of Tsallis entropy to obtain generalized BOBW guarantees, we instead incorporate an additional term into the update of learning rates.
For $\alpha \in (1,2)$, additional $\Theta(1/t)$ term is introduced to ensure $\beta_t=\Omega(\log t)$, which is required for our analysis to address extreme cases where the resampling procedure is repeated excessively even when $w_{t,i_t}$ is sufficiently large.
Combined with the budget of resampling steps $G_t=2\log t$ in CGR II with $\gE_{t,\alpha}$ in (\ref{def: event for decomposition for small alpha}), this choice simplifies the analysis.
For $\alpha \geq 2$, we restrict the learning rate updates so that $\beta_t$ cannot increase too rapidly between rounds, which guarantees Lemma~\ref{lem: pt and qt}.

In particular, with these definitions, we obtain the following results, which provide the same structure to those for FTRL with SPM learning rates in (\ref{eq: SPM for FTRL}).
\begin{lemma}\label{lem: regret decomposition all}
With $\beta_t$ in (\ref{def: lr in MAB}) and $h_t, z_t$ defined in (\ref{def: hz in MAB}), Algorithm~\ref{alg: FTPL MAB} with $\alpha>1$ satisfies
    \begin{align*}
    \Reg_{\textnormal{FTPL}}(T) \leq \begin{cases}
     \sum_{t=1}^{T-1}  \gO\qty(\frac{z_t}{\beta_t})+ t_0(\alpha,K) + \gO\qty(\frac{\alpha^3}{(\alpha-1)^2}\sqrt{K}), &\text{if } \alpha \geq 2 ,\\
        \sum_{t=1}^{T-1} \gO\qty(\frac{z_t}{\beta_t})+ \gO\big(\frac{\alpha^2K^{1/\alpha}}{(\alpha-1)} \log T\big)+ \gO\qty(\frac{\alpha^3}{(\alpha-1)^2}\sqrt{K})+2 , &\text{if } \alpha \in (1,2),
    \end{cases}
\end{align*}
where $t_0(\alpha,K) = \gO(\alpha^2 K^2 \log^2(\alpha K))$.
\end{lemma}
Although Lemma~\ref{lem: regret decomposition all} includes $t_0(\alpha, K)$ term for $\alpha \geq 2$, which depends on $\alpha$ and $K$, the dependence on $K$ can be eliminated.
Specifically, when $i_t = j_t$ and $w_{t,i_t}$ is sufficiently large, setting $G_t = \Theta(\log t)$ instead of $K \log t$ can eliminate such dependency.
This condition can be verified in the same way as the case $\alpha \in (1,2)$ by introducing suitable events $\mathcal{E}_{t,\alpha}$ as in (\ref{def: event for decomposition for small alpha}).

Therefore, the main leading term of the regret upper bound becomes the first term $\sum_t z_t/\beta_t$.
By appropriately adapting the arguments in Lemmas~9 and~10 of \citet{pmlr-v247-ito24a}, we obtain the following lemma.
\begin{lemma}\label{lem: spm regret}
    For Algorithm~\ref{alg: FTPL MAB} with $\alpha >1$ and $\beta_t$ defined in (\ref{def: lr in MAB}), it holds that
    \begin{equation*}
         \sum_{t=1}^T \frac{z_t}{\beta_t} \leq \gO\qty(\min\qty{\sqrt{\log T \sum_{t=1}^T h_tz_t }+ \sqrt{h_{\max}z_{\max}}, \sqrt{h_{\max}\sum_{t=1}^T z_t}}) + \gO\qty(\frac{\alpha z_{\max}}{\beta_1}),
    \end{equation*}
    where $z_{\max}$ and $h_{\max}$ satisfy $z_t \leq z_{\max}$ and $h_t \leq h_{\max}$ for all $t\in \sN$.
\end{lemma}
Therefore, in the adversarial regime, it is sufficient to provide the bound of $h_{\max}\sum_{t}z_t$.
For the adversarial regime with $(\Delta, C, T)$ self-bounding constraint, the key step is to show $h_tz_t \leq \omega(\Delta)\cdot \inp{w_{t+1}}{\Delta}$ for some constants $\omega(\Delta)$, although we need to control additional terms due to the use of surrogate probabilities.
\begin{theorem}\label{thm: bobw mab}
    For the $K$-armed bandit problem, Algorithm~\ref{alg: FTPL MAB} with $\beta_t$ in (\ref{def: lr in MAB}), and any $\alpha >1$ achieves the following bounds simultaneously.
    In the adversarial regime, we have
    \begin{align*}
        \Reg(T) \leq \gO\qty(\frac{\alpha^2}{\alpha-1}\sqrt{KT}).
    \end{align*}
    In the adversarial regime with $(\Delta, C, T)$ self-bounding constraint, we have
    \begin{align*}
       \Reg(T) &\leq \gO\qty(\omega(\Delta)\log T + \sqrt{C\omega(\Delta)\log T}+ c(\alpha,K)),
    \end{align*}
    where $\omega(\Delta) = \gO\qty(\frac{\alpha^4}{(\alpha-1)^2} \frac{K}{\Delta_{\min}})$ and $c(\alpha,K)$ denotes the constant depending on $\alpha, K$.
\end{theorem}
To the best of our knowledge, this is the first result establishing BOBW guarantees for FTPL with Fr\'echet-type perturbations for general $\alpha > 1$.
While our analysis recovers the adversarial results of \citet{lee2024follow}, one limitation is that it does not recover the optimal gap-dependent bound $\omega(\Delta) = \gO(\sum_{i\ne i^*}1/\Delta_i)$ when $\alpha = 2$, in contrast to the results of \citet{pmlr-v247-ito24a}.

This gap is due to the use of surrogate probabilities $p_t, q_t$, which, unlike $w_t$, do not necessarily lie in the probability simplex.
The suboptimal dependence on $\Delta$ comes from the worst case where we cannot provide a lower bound on $w_{t}$ in terms of $p_{t}$ without the dependence on $K$.
A representative example is when $\hat{\uL}_{t,i}=0$ for all $i$, where $w_{t,i}=1/K$ but $p_{t,i}=1/\sigma_{t,i}$ for all $i$.
Although it is possible to recover the optimal $\Delta$ dependence by allowing such $K$-dependent bounds, doing so would introduce additional $K$ term in the current bounds.
Therefore, there is room for improvement both in the choice of surrogate probabilities and in the current BOBW analysis of FTPL.
In particular, it may be possible to avoid the use of overly conservative surrogates.

To compare our results with those obtained under explicit probability computation in the FTRL framework, we recall the known relationship between $\gamma$-Tsallis entropy and Fr\'echet-type perturbations with shape $\alpha$, where the correspondence $\alpha \approx 1 / (1 - \gamma)$ has been observed~\citep{kim2019, lee2025}.
Under this correspondence, the quantity $\omega(\Delta)$ in \citet{pmlr-v247-ito24a} can be seen as a bound of order $\gO(\frac{\alpha^2K}{(\alpha-1)\Delta_{\min}})$, which shows a more favorable dependence on $\alpha$ than our result.
Such factor is again due to the use of surrogate probabilities in (\ref{def: surrogate}), where the worst case summations such as $\sum_{n=2}^K n^{-1/\alpha}$ introduce additional $\alpha$-dependent terms that do not appear when working directly with the probability vectors on the simplex.

From a computational perspective, our approach can be more efficient than FTRL methods, especially those with hybrid regularizers that cannot be reduced to one-dimensional optimizations~\citep{jin2023improved, pmlr-v247-ito24a}.
In particular, Algorithm~\ref{alg: FTPL MAB} follows the standard FTPL with i.i.d.~Pareto perturbations, where the only difference is in the design of learning rates.
The dominant computational complexity is from CGR~II and sorting $\hat{L}_t$, leading to an average complexity of $\gO(K\log K)$.
Indeed, \citet{chen2025geometric} showed that FTPL with CGR~II runs faster than FTRL with Tsallis entropy, even though it can be computed efficiently.
By avoiding both hybrid regularization and convex optimization, and using CGR~II-biased algorithm, our policy remains computationally efficient while preserving the desired regret guarantees in terms of $K$ and $T$ for general $\alpha>1$.


\section{Application of SPM learning rates for FTPL in bandit problems with expert advices}\label{sec: contextual}
\begin{algorithm}[t]
   \caption{FTPL with geometric resampling and SPM learning rates for contextual bandits}
   \label{alg: FTPL contextual}
   \SetAlgoVlined
   \DontPrintSemicolon
   \SetKwInOut{Input}{Input}
   \Input{$K, N \in \sN$, $\alpha \geq 2$, $\beta_1>0$, $\hat{L}_1 = 0$.}
   \For{$t=1,2,\ldots$}{
      Sample $r_t = (r_{t,1}, \ldots, r_{t,K})$ i.i.d. from the Pareto distribution with shape $\alpha$ in (\ref{def: f and F of Lomax}).
      
      Select $i_t \in \argmin_{i\in[K]} \{\hat{L}_{t,i} - \beta_tr_{t,i}\}$ and observe advices $\{\pi_{t,i}\}_{i\in [K]}$ 

      Select $a_t \sim \pi_{t,i_t}$, observe $\ell_{t,i_t}$, and set $M_t := 0$.
      
      \Repeat{$a_t = a'$.}{
         $M_t:= M_t+1$.  \tcp*{Geometric resampling}
         
         Sample $r_t' = (r_{t,1}', \ldots, r_{t,K}')$ i.i.d. from the Pareto distribution with shape $\alpha$ in (\ref{def: f and F of Lomax}).

         Select $i' \in \argmin_{i\in[K]} \{\hat{L}_{t,i} - \beta_t r_{t,i}'\}$ and $a'\sim \pi_{t,i'}$.
      }
     Set $\hat{\ell}_{t,i} =M_t \pi_{t,i,a_t} \ell_{t,a_t}$ for all $i \in [K]$ and update $\hat{L}_{t+1} = \hat{L}_t + \hat{\ell}_t$.
     
     Set $\beta_{t+1}$ by the update rule of (\ref{def: lr in MAB}) based on $z_t, h_t$ in (\ref{def: hz in contextual}) and $q_t$ in (\ref{def: surrogate}). 
   }
\end{algorithm}
In this section, we extend the SPM learning rates for FTPL to the bandits with expert advice, also known as contextual bandit settings, where there are $K$ experts and $N$ arms~\citep{auer2002nonstochastic,dann2023blackbox}. 
At each round, expert $i\in [K]$ provides an advice distribution $\pi_{t,i} \in \gP_N$ over $N$ arms, where $\gP_N$ denotes the $(N-1)$-dimensional probability simplex.
The agent selects an expert $i_t \in [K]$ according to FTPL rule in (\ref{def: ftpl it}), observes the advice from all the expert $\{\pi_{t,i}\}_i$, and then selects $a_t\in [N]$ following the distribution of $\pi_{t,i_t}$. 
Let $\tilde{\ell}_{t,i} = \E[\sum_{a=1}^K\pi_{t,i,a} \ell_{t,a}]$ denote the expected loss of the expert $i$ at round $t$.
Then, the regret with respect to the best experts $i^*$ is given by
\begin{align*}
    \Reg(T) = \E\qty[\sum_{t=1}^T \sum_{a=1}^N \pi_{t,i_t,a}\ell_{t,a}] - \min_{i\in [K]}\E\qty[\sum_{t=1}^T \sum_{a=1}^N \pi_{t,i,a}\ell_{t,a}] =  \E\qty[\sum_{t=1}^T \tilde{\ell}_{t,i_t}] - \E\qty[\sum_{t=1}^T \tilde{\ell}_{t,i^*}].
\end{align*}
Let $w_{t,i}$ be the probability that FTPL rule in (\ref{def: ftpl it}) selects an expert $i$ for given $\hat{L}_t$ in contextual bandits and $P_t \in \gP_N$ denote the marginal distribution over arms given the expert-selection probability $w_t$ and the advice $\{\pi_{t,i}\}_{i}$, i.e., $P_{t,a}=\sum_{i=1}^K w_{t,i} \pi_{t,i,a}$.
Since we observe only the loss from the selected arm $a_t$, one can consider the IW estimator of $\tilde{\ell}_{t,i}$, given by $\ell_{t,a_t}\pi_{t,i,a_t}/ P_{t,a_t}$.
Note that, although we choose only one expert $i_t$, the loss estimators can be constructed for all experts since we observe $\pi_{t,i,a_t}$ from all experts $i$, in contrast to multi-armed bandits.

Since $P_{t,a_t}$ cannot be computed in the general FTPL framework, we consider the original geometric resampling, which samples $i_t'$ with resampled perturbations and $a_t' \sim \pi_{t,i,i_t'}$ until the same $a_t$ is selected.
Then, we construct an unbiased estimator $\hat{\ell}_{t,i} = M_t\ell_{t,a_t}\pi_{t,i,a_t}$, where $M_t$ denotes the number of resampling steps at round $t$.
The policy for contextual bandit is given in Algorithm~\ref{alg: FTPL contextual}.

In general, the same regret decomposition can be obtained as the standard multi-armed bandits given in~(\ref{eq: Reg FTPL}).
The main difference arises in the analysis of the stability term due to the difference in estimator $\hat{\ell}_t$.
Nevertheless, most of the techniques developed in Section~\ref{sec: mab}, as well as those from previous BOBW analysis can extend to contextual bandits with minor modifications.
For the sake of analysis, we impose a mild assumption on the advice distributions $\pi_{t,i,a}$, as follows.
\begin{assumption}\label{asm: on advice}
    For all $t\in \sN$, $i \in [K]$, and $a \in [N]$, $\pi_{t,i,a} \in [\nu, 1] \cup \{0\}$, i.e., the advice probability is uniformly bounded from below by some constants $\nu\in (0,1/N)$ if it is not zero. 
\end{assumption}
Note that Algorithm~\ref{alg: FTPL contextual} does not require the information of $\nu$, which implies that Assumption~\ref{asm: on advice} is used only in the analysis.
For the stability term, we have the following results, which corresponds to Lemma~\ref{lem: stab and pen} for the standard multi-armed bandits in Appendix~\ref{app: pt discuss}.
\begin{lemma}\label{lem: stability in contextual overall}
    For any $t \in \sN$, Algorithm~\ref{alg: FTPL contextual} with $\alpha \geq 2$ satisfies that 
    \begin{equation*}
       \E\qty[\inp{\hat{\ell}_t}{\phi(\eta_t\hat{L}_t) - \phi(\eta_t \hat{L}_{t+1})}\middle | \hat{L}_t] \leq \gO\qty(\frac{\alpha N}{\beta_t}\max_{i \ne j_t}p_{t,i}^{1/\alpha})  + g_t(\alpha;\nu),
    \end{equation*}
    where $p_t$ in (\ref{def: surrogate}) and $g_t(\alpha)$ is a function satisfying $\sum_t g_t(\alpha;\nu) = \gO(\alpha^2/\nu)$ if Assumption~\ref{asm: on advice} holds.
\end{lemma}
Based on this observation, we adopt the learning rate $\beta_t$ defined in~(\ref{def: lr in MAB}) with following $z_t$ and $h_t$:
\begin{equation}\label{def: hz in contextual}
    z_t =\alpha  N\max_{i\ne j_{t+1}} q_{t+1,i}^{1/\alpha}, \text{ and } h_t = \frac{\alpha}{\alpha-1} \sum_{i\ne j_{t+1}} q_{t+1,i}^{1-1/\alpha},
\end{equation}
where $h_t$ coincides with that of multi-armed bandits in~(\ref{def: hz in MAB}). 
This choice is natural, as the penalty term remains unchanged.
With these choices, Algorithm~\ref{alg: FTPL contextual} obtains the following BOBW guarantee.
\begin{theorem}\label{thm: bobw contextual}
    For contextual bandits of $N$ arms with $K$ experts under Assumption~\ref{asm: on advice}, Algorithm~\ref{alg: FTPL contextual} with $\beta_t$ in (\ref{def: lr in MAB}) with $h_t,z_t$ in (\ref{def: hz in contextual}) and $\beta_1 =\gO(\alpha N)$, and any $\alpha \geq 2$ achieves the following bounds simultaneously.
    In the adversarial regime, we have
    \begin{align*}
        \Reg(T) \leq \gO\qty(\sqrt{\frac{\alpha^3}{\alpha-1}NK^{1/\alpha} T}).
    \end{align*}
    In the adversarial regime with a $(\Delta, C, T)$ self-bounding constraint, we have
    \begin{align*}
       \Reg(T) &\leq \gO\qty(\omega(\Delta)\log T + \sqrt{C\omega(\Delta)\log T} +\sqrt{\frac{\alpha^3 NK^{1/\alpha}}{\alpha-1}} + \alpha^2/\nu),
    \end{align*}
    where $\omega(\Delta) = \gO\qty(\frac{\alpha^3}{\alpha-1} \frac{NK^{1/\alpha}}{\Delta_{\min}})$.
\end{theorem}
While our results match those of \citet{pmlr-v247-ito24a} in terms of $N, K$ and $T$, the additional $\sqrt{\alpha}$ factor in our bounds introduces an extra $\log K$ dependence when $\alpha$ is set to minimize $\alpha^2 K^{1/\alpha}$ term.
Specifically, setting $\alpha = \Theta(\log K)$ provides an adversarial regret of $\gO(\sqrt{NT\log^2 K })$ and a regret of $\gO(N\log^2 K \log T/\Delta_{\min})$ under self-bounding constraints.
In contrast, the corresponding bounds in \citet{dann2023blackbox} and \citet{pmlr-v247-ito24a}, which use explicit probabilities, are $\gO(\sqrt{N T\log K})$ in the adversarial regime and $\gO(N\log K \log T/\Delta_{\min})$ for adversarial regime with self-bounding constraints.
These results show one limitation of surrogate-based approaches, especially when the shape parameter, the parameter of the distribution, is determined by the problem-dependent constants.

In the contextual bandit setting, our bounds do not improve upon the best previously known BOBW guarantees.
Nevertheless, a key strength of our approach lies in its generality through the use of SPM learning rates: both Algorithm~\ref{alg: FTPL MAB} for multi-armed bandits and Algorithm~\ref{alg: FTPL contextual} for contextual bandits share the same selection rule including perturbation distributions, differing only in the loss estimation and update of learning rate, which naturally reflects the differences between the problem settings.\footnote{While Algorithm~\ref{alg: FTPL MAB} employs CGR II-biased algorithm, the similar regret bounds can be obtained by using the original GR when $\alpha\geq 2$.
This observation indicates that Algorithms~\ref{alg: FTPL MAB} and~\ref{alg: FTPL contextual} indeed share the same algorithmic structure.}
Moreover, our analysis expresses regret in terms of surrogate quantities $p_t, q_t$ for both settings, making the framework largely agnostic to the specific bandit model, as achieved in FTRL framework~\citep{pmlr-v247-ito24a, nguyen2025data}.
Therefore, we expect the similar analytical approach can be extended beyond the classical bandit setting, where similar regret decompositions and bounds can be obtained with minor modifications.

\section{Conclusion and future work}
We proposed adaptive learning rates for FTPL based on the SPM principle, which was originally developed for the FTRL framework with explicit arm-selection probabilities.
We showed that the SPM methodology can be extended to FTPL by appropriately choosing surrogate probabilities, and established the BOBW guarantees for FTPL with general perturbation parameters both in standard multi-armed bandits and in bandit problems with expert advice.
Although our analysis is grounded in the FTPL framework, we expect that our approach, the use of surrogate probabilities, could be applicable to other algorithmic frameworks as well.
In particular, surrogate probabilities may offer a way to replace explicit probability computations in settings where FTRL-based methods are applicable, potentially leading to more computationally efficient alternatives.

These observations open several directions for future work on efficient BOBW policies in settings not covered in this paper, such as graph bandits, linear bandits, partial monitoring, and bandits with heavy-tailed losses.
In the heavy-tailed setting, for example, surrogate probabilities may be used in place of arm-selection probabilities to define adaptive threshold values for loss estimators~\citep{huang2022adaptive}.
Recently, \citet{zhao2025heavy} further employed both the adaptive thresholding and the SPM learning rate in the FTRL framework to obtain the BOBW guarantee in heavy-tailed linear bandits.
However, handling negative losses introduces additional challenges in the current BOBW analysis of FTPL, since large negative loss can make a previously suboptimal arm to become the best arm after being selected, a case that does not occur in settings with nonnegative losses.

In addition, while some FTRL policies employ arm-dependent learning rates~\citep{jin2023improved, nguyen2025data}, it remains unclear whether such designs can be extended to FTPL. 
From the FTPL perspective, arm-dependent learning rates correspond to FTPL with arm-dependent perturbation distributions and arm-independent learning rates, that is, FTPL with non-i.i.d.~perturbations.
Clarifying whether comparable guarantees, especially in the analysis of the stability term, can be obtained in this setting is an interesting open question.

\bibliographystyle{plainnat}
\bibliography{ref}

\newpage
\appendix


\section{Details on conditional geometric resampling}\label{app: CGR}
Here, we provide more details on CGR II.
The main idea of CGR is not to generate new perturbations that clearly violates the termination condition.
To do this, it restricts the distributions of perturbations in the resampling steps by considering certain necessary condition $\gA_{t}$.
In CGR II, we used the following condition
\begin{equation}\label{def: event for CGR}
    \gA_t = \qty{r_{t,i_t}' = \max_{i  : \sigma_{t,i}\leq \sigma_{t,i_t}}r_{t,i}', \, r_{t,i_t}' \geq \eta_t \hat{\uL}_{t,i_t}},
\end{equation}
where we force the new perturbation from the actually selected arm at round $t$ to be larger than those of arms with smaller cumulative loss estimates and also to be larger than $\eta_t \hat{\uL}_{t,i_t}$.
This is equivalent to sample $r_{t,i_t}'$ with the distribution whose distribution function is
\begin{equation}\label{eq: F of it CGR}
    F_{i_t}(x; \eta_t\hat{\uL}_{t,i_t}) = \frac{F^{\sigma_{t,i_t}}(x)-F^{\sigma_{t,i_t}}(\eta_t \hat{\uL}_{t,i_t})}{1-F^{\sigma_{t,i_t}}(\eta_t \hat{\uL}_{t,i_t})} , \, x \geq \eta_t \hat{\uL}_{t,i_t}.
\end{equation}
After sampling $r_{t,i_t}'$, CGR II samples the perturbations for $i \in \{j: \sigma_{t,j} < \sigma_{t,i_t}\}$ i.i.d.~from the truncated distribution over $[0, r_{t,i_t}')$ whose distribution function is given by
\begin{equation}\label{eq: F of other CGR}
    F_i(x; r_{t,i_t}') = F(x)/F(r_{t,i_t}'), \, x\in [0, r_{t,i_t}'].
\end{equation}
For the remaining arms, $\{j: \sigma_{t,j} > \sigma_{t,j}\}$, the perturbations are sampled independently from the original perturbation distribution.
Then, Algorithm~\ref{alg: FTPL MAB} can be explicitly written as Algorithm~\ref{alg: FTPL MAB real}.

\begin{algorithm}[t]
   \caption{FTPL with conditional geometric resampling II-biased and SPM learning rates}
   \label{alg: FTPL MAB real}
   \DontPrintSemicolon
   \SetAlgoLined
   \SetKwInOut{Input}{Input}
   \Input{$K \in \sN$, $\alpha >1$, $\beta_1>0$, $\hat{L}_1 = 0$.}
   \For{$t=1, 2,\ldots$}{
        Sample $r_t = (r_{t,1}, \ldots, r_{t,K})$ i.i.d. from the Pareto distribution with shape $\alpha$ in (\ref{def: f and F of Lomax}).
        
        Select $i_t \in \argmin_{i\in[K]} \{\hat{L}_{t,i} - \beta_tr_{t,i}\}$ and observe $\ell_{t,i_t}$.

        Find $j_t = \argmin_i \hat{L}_{t,i}$ and set $M_{t}:=0$, and $G_t = K \log t$.

        \textbf{if} $i_t = j_t$, $\alpha \in (1,2)$, and $\I[\gE_{t,\alpha}]=1$ in (\ref{def: event for decomposition for small alpha}) \textbf{then} $G_t = 2\log t$.

        \Repeat{$i_t = \argmin_{i\in[K]} \{\hat{L}_{t,i} - \beta_t r_{i}'\}$ \textnormal{or} $M_t \geq G_t$}{
            $M_t := M_t +1$. \tcp*{CGR II-biased}
            
            Sample $r_{t,i_t}'$ from truncated distribution over $[\eta_t\hat{\uL}_{t,i_t},\infty)$ as in (\ref{eq: F of it CGR}).

            Sample $\{r_{t,i}: \sigma_{t,i} < \sigma_{t,i_t} \}$ i.i.d. from truncated distribution over $[0, r_{t,i_t}']$ as in (\ref{eq: F of other CGR}).

            Sample $\{r_{t,i}: \sigma_{t,i} > \sigma_{t,i_t}\}$ i.i.d.~from Pareto with shape $\alpha$ in (\ref{def: f and F of Lomax}).
        }
        Set $\hat{\ell}_{t,i_t} =M_t \sigma_{t,i_t}/(1-F^{\sigma_{t,i_t}}(\eta_t\hat{\uL}_{t,i_t}))$ and update $\hat{L}_{t+1} = \hat{L}_t + \hat{\ell}_t$. 
        
        Set $\beta_{t+1}$ by the update rule of (\ref{def: lr in MAB}) based on $z_t, h_t$ in (\ref{def: hz in MAB}) and $q_t$ in (\ref{def: surrogate}).
   }
\end{algorithm}


\section{Intuition and difficulties behind the surrogate-based learning rates}\label{app: pt discuss}
Here, we discuss the intuition behind on the choice of $q_{t,i}$, and definitions of $z_t$ and $h_t$.
In the standard multi-armed bandits, we can obtain the following bounds by applying the results of \citet{lee2024follow}.
\begin{lemma}\label{lem: stab and pen}
    For any $t \in \sN$, FTPL with Pareto distribution with $\alpha >1$ and monotonically decreasing learning rates $\eta_t=1/\beta_t$ satisfies that for any $i \in [K]$
    \begin{align*}
        \E\qty[\hat{\ell}_{t,i}(\phi_i(\eta_t \hat{L}_t) - \phi_i(\eta_t(\hat{L}_t+ \hat{\ell}_{t,i}e_i)) \middle | \hat{L}_t] &\leq \frac{2e\alpha}{\beta_t} p_{t,i}^{\frac{1}{\alpha}}, \\
        \E\qty[\I[i_t =i] r_{t,i}] &\leq \frac{2\alpha}{\alpha-1}p_{t,i}^{1-\frac{1}{\alpha}}.
    \end{align*}
\end{lemma}
This lemma implies that we can rewrite the regret of FTPL in (\ref{eq: Reg FTPL}) in terms of $p_{t,i}$ as follows:
\begin{equation}\label{eq: reg FTPL with naive SPM}
    \Reg_{\text{FTPL}}(T) \lesssim \gO\qty(\sum_{t=1}^T \frac{\sum_{i=1}^K p_{t,i}^{\frac{1}{\alpha}}}{\beta_t} + \sum_{t=1}^T (\beta_{t+1} - \beta_t)\sum_{i\ne i^*} p_{t+1,i}^{1-\frac{1}{\alpha}}).
\end{equation}
Since $i^*$ is unknown, one may recover a formulation analogous to (\ref{eq: Reg FTRL}) by setting $z_t$ and $h_t$ by
\begin{equation}\label{eq: naive SPM MAB}
    z_t' = \alpha \sum_{i=1}^K p_{t,i}^{1/\alpha} \text{ and }h_t' =  \frac{\alpha}{\alpha-1} \sum_{i=1}^K p_{t,i}^{1-1/\alpha},
\end{equation}
which have a structure very similar to those of FTRL in (\ref{def: hz of FTRL}).
However, directly using these quantities to design SPM learning rates as in (\ref{eq: SPM for FTRL}) incurs two technical difficulties.
Firstly, as briefly discussed in Section~\ref{sec: FTRL spm}, updating $\beta_{t+1}$ requires the value of $h_{t+1}'$, whose definition depends on $p_{t+1,i}$ and hence on $\beta_{t+1}$.
In principle, one can compute such $\beta_{t+1}$ even with this $p_{t+1,i}$ by iteratively solving this circular dependence.
However, doing so would introduce undesirable computational cost, which loses the motivation to use FTPL.
To justify the use of $h_{t+1}'$, one therefore need to show $p_{t+1,i}=\gO(p_{t,i})$ and then replace $h_{t+1}'$ with $h_t'$ as done in FTRL~\citep{pmlr-v247-ito24a, nguyen2025data}.
Secondly, the definitions of $z_t'$ and $h_t'$ involve summation over all arms, which introduces an additional obstacle when relating these quantities to $w_{t}$ in adversarial regime with self-bounding constraints.
In particular, $p_{t,j_t}$ is always $1$ by the construction, whereas, in the FTRL framework, the corresponding term can be replaced with $\min(1-w_{t,j_t}, w_{t,j_t})$ term~\citep{pmlr-v247-ito24a}.

To address the difficulties above, our definitions of $z_t$ and $h_t$ in (\ref{def: hz in MAB}) allocate the contribution of the current best arm $j_{t+1}$ into the summation over the other arms as well as it uses $q_t$ that are computable without access to $\beta_{t+1}$.
This allow us to construct an update rule for $\beta_{t+1}$ that depends only on quantities available at the round $t$.
Nevertheless, to justify this construction, we require the following relationship between $q_{t,i}$ and $p_{t,i}$.
\begin{lemma}[Formal version of Lemma~\ref{lem: pt and qt}]\label{lem: formal pt and qt}
    It holds that $q_{t,i}\leq p_{t,i} \leq  2 q_{t,i}$ for $t \geq 3$ when $\alpha \in (1, 2)$, and for $t\in \sN$ when $\alpha \geq 2$. 
\end{lemma}
Lemma~\ref{lem: formal pt and qt} implies that we can still utilize the results of Lemma~\ref{lem: stab and pen} in terms of $q_{t,i}$.
Moreover, it clearly shows a correspondence between $z_{t+1}', h_{t+1}'$ in (\ref{eq: naive SPM MAB}), which are not efficiently computable at round $t$, and $z_{t}, h_{t}$ in (\ref{def: hz in MAB}), which can be easily computed at the end of round $t$.
Therefore, the regret upper bound in (\ref{eq: reg FTPL with naive SPM}) can be roughly expressed as follows, where we ignore the term related to $j_t$ in $h_t', z_t'$ and some constants for the purpose of illustration.
\begin{align*}
    \sum_{t=1}^T \frac{z_t'}{\beta_t} + (\beta_{t+1} - \beta_t)h'_{t+1}  &=\sum_{t=1}^{T-1} \qty(\frac{z_{t+1}'}{\beta_{t+1}} + (\beta_{t+1}-\beta_t)h_{t+1}') + \frac{z_1'}{\beta_1} + (\beta_{T+1} - \beta_T)h_{T+1}' \\
    &\lesssim \sum_{t=1}^{T-1} \qty(\frac{2^{\frac{1}{\alpha}} z_{t}}{\beta_{t+1}} + 2^{1-\frac{1}{\alpha}}(\beta_{t+1}-\beta_t)h_{t})  + \frac{z_1'}{\beta_1}  + \frac{z_T}{\beta_T h_T}h_{T+1}'\\\
    &\lesssim \sum_{t=1}^{T-1} \qty(\frac{2^{\frac{1}{\alpha}} z_{t}}{\beta_{t}} + 2^{1-\frac{1}{\alpha}}(\beta_{t+1}-\beta_t)h_{t}) \approx \sum_{t=1}^{T-1} 2 \frac{z_t}{\beta_t},
\end{align*}
which roughly recovers the formulation in Lemma~\ref{lem: regret decomposition all}.
The second line follows from the definition of $\beta_t$ in (\ref{def: lr in MAB}).
In third line, we ignore constant terms since the last term becomes negligible for large $T$ due to $\beta_T = \Omega(\log T)$.
While it may be possible to avoid the above difficulties directly using $p_t$, we find it technically more complicated due to the use of surrogate, whereas our construction leads to a considerably simpler and more convenient analysis.

\section{Proofs of Lemmas in multi-armed bandits}\label{app: mab}
In this section, we provide the proofs omitted in Section~\ref{sec: mab}.

\subsection{Proof on the relationship between surrogate probability and arm-selection probability}\label{app: ineq}
Here, we show that
\begin{equation*}
    w_{t,i} \leq p_{t,i}, \, \forall i \in [K], \, t \in \sN.
\end{equation*}
\begin{proof}
    Since FTPL plays an arm according to (\ref{def: ftpl it}), where all the perturbations are generated independently from the identical distribution, it is clear that the arm with the smallest cumulative loss $\hat{L}_{t,\cdot}$ will be of the highest probability to be selected.
    When $\hat{L}_{t,i}$ is $\sigma_{t,i}$-th smallest, then its arm-selection probability should be smaller than $1/\sigma_{t,i}$ since there exist $\sigma_{t,i}-1$ arms with smaller cumulative loss.

    The left $\hat{L}_t$ dependent bounds can be directly obtained by the definition of $w_{t,i}$ in (\ref{def: ftpl wt}) as follows.
    \begin{align*}
        w_{t,i} &= \int_0^\infty \frac{\alpha}{(z+\eta_t\hat{\uL}_{t,i}+1)^{\alpha+1}} \prod_{j \ne i} \qty(1- \frac{1}{(z+\eta_t \hat{\uL}_{t,i}+1)^\alpha})\dd z \\
        &\leq \int_0^\infty \frac{\alpha}{(z+\eta_t\hat{\uL}_{t,i}+1)^{\alpha+1}} \dd z = \frac{1}{(1+\eta_t\hat{\uL}_{t,i})^\alpha},
    \end{align*}
    which concludes the proof.
\end{proof}

\subsection{Proof of Lemma~\ref{lem: stab and pen}}\label{app: stab and pen proof}
While the overall proofs follow the results in previous FTPL analysis with Pareto perturbation~\citep{pmlr-v201-honda23a, lee2024follow, kim2025follow}, we provide the detailed proofs since we use the slightly tighter results than their presented results.
\begin{proof}
Let us start from proving the first results, which is the bound for the stability term.
    \paragraph{Stability analysis.}
    For any $\lambda\in \sR^K$, define
    \begin{align*}
        \phi_i'(\lambda)  := \frac{\partial \phi_i(\lambda)}{\partial \lambda_i} &= \int_0^\infty f'(z+\ul_i) \prod_{j \ne i} F(z+\ul_j) \dd z \\
        &=  \int_0^\infty \frac{-\alpha(\alpha+1)}{(z+\ul_i+1)^{\alpha+2}} \prod_{j \ne i} \qty(1- \frac{1}{(z+\ul_i+1)^\alpha})\dd z. \numberthis{\label{eq: decreasing phi'}}
    \end{align*}
    Then, by definition, we can obtain
    \begin{align*}
        \phi_i(\eta_t \hat{L}_t) - \phi_i(\eta_t(\hat{L}_t+ \hat{\ell}_{t,i}e_i) &\leq \int_0^{\eta_t \hat{\ell}_{t,i}} -\phi_i'(\eta_t\hat{L}_t + xe_i) \dd x\\ 
        &\leq \int_0^{\eta_t \hat{\ell}_{t,i}} -\phi_i'(\eta_t\hat{L}_t) \dd x \tag{$\because$ decreasing w.r.t. $\lambda_i$ by (\ref{eq: decreasing phi'})}\\ 
        &\leq-\eta_t \hat{\ell}_{t,i} \phi_i'(\eta_t \hat{L}_t).
    \end{align*}
    Therefore, we have
    \begin{align*}
         \E\qty[\hat{\ell}_{t,i}(\phi_i(\eta_t \hat{L}_t) - \phi_i(\eta_t(\hat{L}_t+ \hat{\ell}_{t,i}e_i)) \middle | \hat{L}_t] &\leq \E\qty[-\eta_t \hat{\ell}_{t,i}^2 \phi_i'(\eta_t \hat{L}_t) \middle | \hat{L}_t ] \numberthis{\label{eq: stab link mab}}\\
         &= \E\qty[-\I[i_t =i]\eta_t \ell_{t,i}^2 M_t^2 \phi_i'(\eta_t \hat{L}_t) \middle | \hat{L}_t ] \\
         &\leq  \E\qty[-\I[i_t =i] \eta_t \ell_{t,i}^2 \frac{2}{w_{t,i}^2}\phi_i'(\eta_t \hat{L}_t) \middle | \hat{L}_t ] \numberthis{\label{eq: varaince of CGR}}\\
         &\leq  \E\qty[\eta_t \frac{-2\phi_i'(\eta_t \hat{L}_t)}{w_{t,i}} \middle | \hat{L}_t ], \tag{$\because \ell_t \in [0,1]^K$}
    \end{align*}
    where (\ref{eq: varaince of CGR}) follows from $\E[M_t^2|\hat{L_t}, i_t] \leq 2/w_{t,i_t}^2$.
    Here, Lemma 9 in \citet{lee2024follow} shows that $-\phi_i'(\lambda)/\phi_i(\lambda)$ is monotonically increasing with respect to $\lambda_j$ for any $j\ne i$.
    Therefore, we have
    \begin{align*}
        \frac{-\phi_i'(\eta_t \hat{L}_t)}{\phi_i(\eta_t \hat{L}_t)} \leq \frac{-\phi_i'(\eta_t L^*)}{\phi_i(\eta_t L^*)}, \text{ where } L^*_j = \begin{cases}
            L_i, &\text{if } \sigma_{t,j} \leq \sigma_{t,i}, \\
            \infty, &\text{if } \sigma_{t,j}> \sigma_{t,i}.
        \end{cases}
    \end{align*}
    By definition of $L^*$, we have
    \begin{align*}
        \frac{-\phi_i'(\eta_t L^*)}{\phi_i(\eta_t L^*)} &= \frac{\int_0^\infty \frac{\alpha(\alpha+1)}{(z+\eta_t \uL_{t,i}+1)^{\alpha+2}} \qty(1-\frac{1}{(z+\eta_t \hat{\uL}_{t,i}+1)^\alpha})^{\sigma_{t,i}-1} \dd z}{\int_0^\infty \frac{\alpha}{(z+\eta_t \uL_{t,i}+1)^{\alpha+1}} \qty(1-\frac{1}{(z+\eta_t \hat{\uL}_{t,i}+1)^\alpha})^{\sigma_{t,i}-1} \dd z} \\
        &= \frac{(\alpha+1)\int_0^{\frac{1}{(1+\eta_t\hat{\uL}_{t,i})^\alpha}} y^{1/\alpha} (1-y)^{\sigma_{t,i}-1} \dd y}{\int_0^{\frac{1}{(1+\eta_t\hat{\uL}_{t,i})^\alpha}}(1-y)^{\sigma_{t,i}-1} \dd y} \tag{$y= 1/(z+\eta_t\hat{\uL}_{t,i}+1)^\alpha$} \\
        &= (\alpha+1)\frac{B(1/(1+\eta_{t}\hat{\uL}_{t,i})^\alpha; 1+1/\alpha, \sigma_{t,i})}{B(1/(1+\eta_{t}\hat{\uL}_{t,i})^\alpha; 1, \sigma_{t,i})} \tag{incomplete Beta function $B(x;a,b)$}\\
        &\leq (\alpha+1) \frac{e\alpha}{(\alpha+1)} \frac{1}{(1+\eta_t\hat{\uL}_{t,i})} \tag{by (36) of \citet{lee2024follow}}
    \end{align*}
    which concludes the proof for the first term in the stability term.
    For the second term, \citet{lee2024follow} showed that
    \begin{align*}
        \frac{B(x; 1+1/\alpha, \sigma_{t,i})}{B(x; 1, \sigma_{t,i})} &\leq \frac{B(1+1/\alpha, \sigma_{t,i})}{B(1, \sigma_{t,i})} \tag{Beta function $B(a,b)$}\\
        &\leq \frac{2\alpha}{\alpha+1} \Gamma\qty(1+\frac{1}{\alpha}) \frac{1}{\sigma_{t,i}^{1/\alpha}} \tag{Gamma function $\Gamma(a)$}
    \end{align*}
    Therefore, we obtain that
    \begin{align*}
         \E\qty[\hat{\ell}_{t,i}(\phi_i(\eta_t \hat{L}_t) - \phi_i(\eta_t(\hat{L}_t+ \hat{\ell}_{t,i}e_i)) \middle | \hat{L}_t] \leq 2\eta_t \min\qty(\frac{e\alpha}{1+\eta_t \hat{\uL}_{t,i}}, \frac{2\alpha\Gamma(1+1/\alpha)}{\sigma_{t,i}^{1/\alpha}}).
    \end{align*}
    Since $\Gamma(1+1/\alpha) < \Gamma(2)=1$ for $\alpha >1$, it concludes the proof for the stability term.

    \paragraph{Penalty analysis.}
     By definition, we have for any $i\in [K]$ that
    \begin{align*}
        \E\qty[\I[i_t= i] r_{t,i}  \middle | \hat{L}_t] &=\int_0^\infty \frac{\alpha}{(z+\eta_t \hat{\uL}_{t,i}+1)^\alpha} \prod_{j \ne i} \qty(1-\frac{1}{(z+\eta_t\hat{\uL}_{t,j}+1)^\alpha}) \dd z \\
        &\leq \int_0^\infty \frac{\alpha}{(z+\eta_t \hat{\uL}_{t,i}+1)^\alpha} \dd z \\
        &\leq \frac{\alpha}{\alpha-1} \frac{1}{(1+\eta_t\hat{\uL}_{t,i})^{\alpha-1}}.
    \end{align*}
    For the second part of penalty term, since $F(z+\lambda_j)$ is increasing with respect to $\lambda_j$, we have
    \begin{align*}
        \E\qty[\I[i_t= i] r_{t,i} \middle | \hat{L}_t] &= \int_0^\infty \frac{\alpha}{(z+\eta_t \hat{\uL}_{t,i}+1)^\alpha} \prod_{j \ne i} \qty(1-\frac{1}{(z+\eta_t\hat{\uL}_{t,j}+1)^\alpha}) \dd z \\
        &\leq \int_0^\infty \frac{\alpha}{(z+\eta_t \hat{\uL}_{t,i}+1)^\alpha} \qty(1-\frac{1}{(z+\eta_t\hat{\uL}_{t,i}+1)^\alpha})^{\sigma_{t,i}-1} \dd z \\
        &= \int_{0}^{\frac{1}{(1+\eta_t\hat{\uL}_{t,i})^\alpha}} y^{-\frac{1}{\alpha}} (1-y)^{\sigma_{t,i}-1} \dd y \tag{$y = 1/(z+\eta_t\hat{\uL}_{t,i}+1)^\alpha$} \\
        &\leq  \int_{0}^{1} y^{-\frac{1}{\alpha}} (1-y)^{\sigma_{t,i}-1} \dd y \\
        &= B\qty(1-\frac{1}{\alpha},\sigma_{t,i}) = \frac{\Gamma(1-1/\alpha)\Gamma(\sigma_{t,i})}{\Gamma(\sigma_{t,i}+1-1/\alpha)}.
    \end{align*}
    By Gautschi's inequality, we have for any $x>0$ and $s \in (0,1)$ that
    \begin{equation*}
        \frac{\Gamma(x)}{\Gamma(x+s)} < \frac{(x+1)^{1-s}}{x}.
    \end{equation*}
    Since $i\in \sN$ and $\alpha>1$, we have
    \begin{align*}
        \frac{\Gamma(\sigma_{t,i})}{\Gamma(\sigma_{t,i}+1-1/\alpha)} \leq \frac{(\sigma_{t,i}+1)^{\frac{1}{\alpha}}}{\sigma_{t,i}} &= \frac{\sigma_{t,i}^{1/\alpha}}{\sigma_{t,i}} \qty(\frac{\sigma_{t,i}+1}{\sigma_{t,i}})^{\frac{1}{\alpha}} \\
        &\leq 2^{1/\alpha}\frac{1}{\sigma_{t,i}^{1-1/\alpha}}.
    \end{align*}
    While we can show that $\frac{\alpha}{\alpha-1} \leq 2^{\frac{1}{\alpha}}\Gamma(1-1/\alpha)$ for $\alpha >1$, for clear $\alpha$ dependency, we show that $2^{\frac{1}{\alpha}}\Gamma(1-1/\alpha) \leq \frac{2\alpha}{\alpha-1}$ for $\alpha >1$.
    This is equivalent to show that
    \begin{align*}
        2^{\frac{1}{\alpha}} \qty(1-\frac{1}{\alpha}) \Gamma\qty(1-\frac{1}{\alpha}) \leq 2.
    \end{align*}
    By the property of the Gamma function, $\Gamma(x+1)=x\Gamma(x)$, this is also equivalent to show
    \begin{equation*}
        2^{\frac{1}{\alpha}} \Gamma\qty(2-\frac{1}{\alpha}) \leq 2.
    \end{equation*}
    Since $\Gamma(x) \leq 1$ for $x \in [1,2]$ and $2^{1/\alpha} <2$ for $\alpha >1$, the above inequality is valid.
\end{proof}

\begin{remark}
    In (36) of \cite{lee2024follow}, the original inequality includes an additional factor $1/\sigma_{t,i}$ in the upper bound, which does not hold in general.
    However, the authors also used a version of the bound without this factor, which is valid. 
    Therefore, this issue does not affect the correctness of their results; we note it here for completeness.
\end{remark}

\subsection{Proof of Lemma~\ref{lem: pt and qt} (Lemma~\ref{lem: formal pt and qt})}\label{app: pt and qt}
\begin{proof}
    For any $i \in [K]$ and $t \in \sN$, by definition of $\beta_t$ in (\ref{def: lr in MAB}), we have
    \begin{align*}
        \frac{\frac{1}{(1+\eta_{t+1}\hat{\uL}_{t+1,i})^\alpha}}{\frac{1}{(1+\eta_{t}\hat{\uL}_{t+1,i})^\alpha}} = \qty(\frac{1+\eta_{t}\hat{\uL}_{t+1,i}}{1+\eta_{t+1}\hat{\uL}_{t+1,i}})^\alpha = \qty(\frac{\beta_{t+1}}{\beta_t}\frac{\beta_t+\hat{\uL}_{t+1,i}}{\beta_{t+1}+\hat{\uL}_{t+1,i}})^\alpha \leq \qty(\frac{\beta_{t+1}}{\beta_t})^\alpha.
    \end{align*}
    \paragraph{When $\alpha \geq 2$.}
    In this case, by definition of $\beta_t$ in (\ref{def: lr in MAB}), the result directly follows.
    \paragraph{When $\alpha \in (1,2)$.}
    By the update rule of the learning rates, it holds that
    \begin{align*}
        \frac{\beta_{t+1}}{\beta_t} = 1 + \frac{1}{\beta_t}\max\qty(\frac{z_t}{\beta_t h_t}, \frac{4}{(2^{1/\alpha}-1)t}),
    \end{align*}
    where
    \begin{equation*}
        \frac{z_t}{h_t} =  \frac{\sum_{i \ne j_{t+1}} \alpha q_{t+1,i}^{1/\alpha}}{\sum_{i\ne j_{t+1}} \frac{\alpha}{\alpha-1}q_{t+1,i}^{1-1/\alpha}} = (\alpha-1) \frac{\sum_{i \ne j_{t+1}}q_{t+1,i}^{1/\alpha}}{\sum_{i\ne j_{t+1}} q_{t+1,i}^{1-1/\alpha}}.
    \end{equation*}
    When $\alpha \in (1,2)$, $\frac{1}{\alpha} \geq 1-\frac{1}{\alpha}$ holds.
    Since $q_{t,i} \in (0,1/2]$ by its definition for any $i \ne j_{t+1}$ and $t \in \sN$, we have $q_{t+1,i}^{\frac{1}{\alpha}} \leq q_{t+1,i}^{1-\frac{1}{\alpha}}$.
    Therefore, $z_t/h_t \leq 1$ for any $t$, which implies
    \begin{align*}
        \frac{\beta_{t+1}}{\beta_t} \leq 1 + \frac{1}{\beta_t} \max\qty(\frac{1}{\beta_t}, \frac{4}{(2^{1/\alpha}-1)t}).
    \end{align*}
    Here, Lemma~\ref{lem: lr lb in MAB} shows that $\beta_t \geq \frac{4\log t}{(2^{1/\alpha}-1)}$ for $\alpha \in (1,2)$.
    This implies that for $t \geq 3$ 
    \begin{align*}
        \frac{\beta_{t+1}}{\beta_t} \leq 1 + \frac{2^{1/\alpha}-1}{4\log t} \max\qty( \frac{2^{1/\alpha}-1}{4\log t}, \frac{4}{(2^{1/\alpha}-1)t}) &\leq 1 + \frac{(2^{1/\alpha}-1)^2}{16 \log^2 t}+  \frac{1}{t \log t} \\
        &\leq 1 + \frac{1}{16 \log^2 3}+  \frac{1}{3\log 3} \\
        &\leq 1.36.
    \end{align*}
    Since $(1.36)^2 < 2$, we obtain the desired results.
\end{proof}

\subsection{Proof of Lemma~\ref{lem: regret decomposition all}}
To decompose the regret in the desired formulation, we need to relocate the contribution from $j_t$ in the stability to those from $i\ne j_t$.
For this purpose, we need the following lemma, whose proof is given in Appendix~\ref{app: jt decomp MAB}.
\begin{lemma}\label{lem: jt decomposition MAB}
    Algorithm~\ref{alg: FTPL MAB} with shape $\alpha>1$ and $\beta_t$ defined in (\ref{def: lr in MAB}) satisfies that
    \begin{align*}
        \E\qty[\hat{\ell}_{t,j_t}(\phi_{j_t}(\eta_t \hat{L}_t) - \phi_{j_t}(\eta_t(\hat{L}_t+ \hat{\ell}_{t,j_t}e_{j_t})) \middle | \hat{L}_t] \leq \gO\qty(\frac{\sum_{i \ne j_t} \alpha p_{t,i}^{1/\alpha}}{\beta_t}).
    \end{align*}
    The above inequality holds for all $t$ in the case of $\alpha \in (1,2)$ and $t\geq t_0(\alpha,K)$ when $\alpha \geq 2$.
\end{lemma}
\begin{proof}
\textbf{of Lemma~\ref{lem: regret decomposition all} }
By using Lemma 4 of \citet{kim2025follow} with Lemma 18 of \citet{lee2024follow}, it holds that
\begin{multline}\label{eq: decomposition original form}
   \Reg_{\text{FTPL}}(T) \leq \sum_{t=1}^T\E\qty[ \inp{\hat{\ell}_t}{\phi(\eta_t \hat{L}_t) -\phi(\eta_t \hat{L}_{t+1}) }]\\ + \qty(\beta_{t+1} - \beta_t)\E\qty[r_{t+1,i_{t+1}}- r_{t+1,i^*}]  + \beta_1 \qty(\frac{\alpha}{\alpha-1})^2 K^{\frac{1}{\alpha}}.
\end{multline}
For the penalty term, which is the second term of (\ref{eq: decomposition original form}), we have
\begin{equation*}
    \E\qty[r_{t+1,i_{t+1}}- r_{t+1,i^*}] = \E\qty[\E\qty[r_{t+1,i_{t+1}}- r_{t+1,i^*}\middle| \hat{L}_{t+1}]] = \E\qty[\E\qty[r_{t+1,i_{t+1}}- r_{t+1,j_{t+1}}\middle| \hat{L}_{t+1}]]
\end{equation*}
since $j_{t+1}$ is fixed given $\hat{L}_{t+1}$ and $r_{t,i}$s are independently distributed from the identical distribution. 

For the stability term, the first term of (\ref{eq: decomposition original form}), Lemmas~\ref{lem: stab and pen} and~\ref{lem: jt decomposition MAB} imply that
\begin{align*}
   \E\qty[ \inp{\hat{\ell}_t}{\phi(\eta_t \hat{L}_t) -\phi(\eta_t \hat{L}_{t+1}) }\middle| \hat{L}_t] &\leq  \frac{\sum_{i\ne j_t} 2e\alpha p_{t,i}^{\frac{1}{\alpha}}}{\beta_t} + \E\qty[\hat{\ell}_{t,j_t}(\phi_{j_t}(\eta_t \hat{L}_t) - \phi_{j_t}(\eta_t(\hat{L}_t+ \hat{\ell}_t)) \middle | \hat{L}_t]\\
   &\leq \gO\qty(\frac{\alpha \sum_{i \ne j_t} p_{t,i}^{1/\alpha}}{\beta_t}),
\end{align*}
for all $t \in \sN$ if $\alpha \in (1,2)$ and $t \geq t_0(\alpha,K)$ for $\alpha \geq 2$.
Therefore with Lemma~\ref{lem: stab and pen}, (\ref{eq: decomposition original form}) is written
\begin{align*}
    \Reg_{\text{FTPL}}(T) \leq \sum_{t=1}^T \gO\qty(\frac{\alpha \sum_{i \ne j_t} p_{t,i}^{\frac{1}{\alpha}}}{\beta_t}) + (\beta_{t+1}-\beta_t)\sum_{i\ne j_{t+1}} \frac{2\alpha}{\alpha-1} p_{t+1,i}^{1-\frac{1}{\alpha}}  + \beta_1 \qty(\frac{\alpha}{\alpha-1})^2 K^{\frac{1}{\alpha}}.
\end{align*}
Here, we have
\begin{align*}
    \sum_{t=1}^T \frac{\alpha \sum_{i \ne j_t} p_{t,i}^{\frac{1}{\alpha}}}{\beta_t}& = \frac{\alpha \sum_{i \ne j_{1}} p_{1,i}^{1/\alpha}}{\beta_1} + \sum_{t=1}^{T-1} \frac{\alpha \sum_{i \ne j_{t+1}} p_{t+1,i}^{\frac{1}{\alpha}}}{\beta_{t+1}} \\
    &\leq \frac{\alpha \sum_{i \ne j_{1}} p_{1,i}^{1/\alpha}}{\beta_1} + \sum_{t=1}^{T-1} \frac{\alpha \sum_{i \ne j_{t+1}} p_{t+1,i}^{\frac{1}{\alpha}}}{\beta_{t}} \tag{$\because \beta_{t+1} \geq \beta_t$} \\
    &= \frac{\alpha \sum_{n=2}^K n^{-1/\alpha}}{\beta_1} + \sum_{t=1}^{T-1} \frac{\alpha \sum_{i \ne j_{t+1}} p_{t+1,i}^{\frac{1}{\alpha}}}{\beta_{t}}  \tag{$\because \hat{L}_1 =0$} \\
    &\leq \frac{\alpha^2}{\alpha-1} \frac{K^{1-\frac{1}{\alpha}}}{\beta_1} + \sum_{t=1}^{T-1} \frac{\alpha \sum_{i \ne j_{t+1}} p_{t+1,i}^{\frac{1}{\alpha}}}{\beta_{t}},
\end{align*}
which implies that 
\begin{align*}
    \Reg_{\text{FTPL}}(T) &\leq  \sum_{t=1}^{T-1} \gO\qty(\frac{\alpha \sum_{i \ne j_{t+1}}p_{t+1,i}^{1/\alpha}}{\beta_t}) + (\beta_{t+1}-\beta_t) \gO\qty(\frac{\alpha }{\alpha-1}\sum_{i \ne j_{t+1}}p_{t+1,i}^{1-1/\alpha}) \\
    &\hspace{3em}+ \frac{1}{\beta_1} \frac{\alpha^2 K^{1-1/\alpha}}{\alpha-1} + \beta_1 \qty(\frac{\alpha}{\alpha-1})^2 K^{\frac{1}{\alpha}} + \frac{z_T h_{T+1}'}{\beta_T h_T} +\I[\alpha \geq 2] t_0(\alpha,K) 
\end{align*}
Since $\beta_{T+1} -\beta_T = \frac{z_T}{\beta_T h_T} \leq \gO(\alpha K/\beta_T)$, whenever $T$ is sufficiently large, the last term becomes negligible since $\beta_T \geq \log T$ for any $\alpha >1$.
Then, Lemma~\ref{lem: formal pt and qt} (Lemma~\ref{lem: pt and qt}) implies that
\begin{align*}
    \Reg_{\text{FTPL}}(T) &\leq  \sum_{t=1}^{T-1} \gO\qty(\frac{\alpha \sum_{i \ne j_{t+1}}q_{t+1,i}^{1/\alpha}}{\beta_t}) + (\beta_{t+1}-\beta_t) \gO\qty(\frac{\alpha }{\alpha-1}\sum_{i \ne j_{t+1}}q_{t+1,i}^{1-1/\alpha}) \\
    &\hspace{1em}+ \frac{1}{\beta_1} \frac{\alpha^2 K^{1-1/\alpha}}{\alpha-1} + \beta_1 \qty(\frac{\alpha}{\alpha-1})^2 K^{\frac{1}{\alpha}}+\I[\alpha \geq 2] t_0(\alpha,K) + \I[\alpha \in (1,2)]2,
\end{align*}
which concludes the proof.
Note that the additional $\log T$ term for $\alpha \in (1,2)$ comes from the summation of $\frac{4}{(2^{1/\alpha}-1)t} \cdot \frac{\alpha}{\alpha-1} \sum_{i \ne j_{t+1}}q_{t+1,i}^{1-1/\alpha}$.
\end{proof}
\subsection{Proof of Lemma~\ref{lem: jt decomposition MAB}}\label{app: jt decomp MAB}
Before the proof of Lemma~\ref{lem: jt decomposition MAB}, we provide a lower bound on the learning rates $\beta_t$ defined in~(\ref{def: lr in MAB}). Although this bound follows directly from its definition, we include the proof for completeness in Appendix~\ref{app: lr lb}.
\begin{lemma}\label{lem: lr lb in MAB}
    For $\beta_t$ defined in (\ref{def: lr in MAB}) satisfies that
    \begin{equation*}
        \beta_t \geq \begin{cases}
            \sqrt{\beta_1^2 + 2(t-1)}, &\text{if } \alpha \geq 2, \\
            \beta_1 + \frac{4\log t}{(2^{1/\alpha}-1)}, &\text{if } \alpha \in (1,2).
        \end{cases}
    \end{equation*}
\end{lemma}
\begin{proof}
\textbf{of Lemma~\ref{lem: jt decomposition MAB} }
    Define a variable $\xi_\alpha$ as
    \begin{equation*}
        \xi_\alpha = \begin{cases}
            \frac{1}{2}, &\text{if } \alpha \in (1,2), \\
            \frac{1}{(2-2^{-1/(\alpha+1)})^\alpha} , &\text{if } \alpha \geq 2.
        \end{cases}
    \end{equation*}
    Note that $\xi_\alpha \geq 1/2$ for all $\alpha>1$ by definition.
    Then, we consider the event $\gE_{t,\alpha}$, which is
    \begin{equation*}
        \gE_{t,\alpha} = \qty{\sum_{i \ne j_t} \frac{1}{(1+\eta_t\hat{\uL}_{t,i})^\alpha} < \xi_\alpha}.
    \end{equation*}
    Note that by the relationship $w_{t,i}\leq p_{t,i} \leq 1/(1+\eta_t\hat{\uL}_{t,i})^\alpha$, (see Section~\ref{app: ineq} for the first inequality), that
     \begin{equation}\label{eq: wjt lower bound on E}
         \sum_{i \ne j_t} w_{t,i} \leq 
            \xi_\alpha, \text{ and } w_{t,j_t} \geq 
            1-\xi_\alpha.
     \end{equation}
     Then, we consider the case $\gE_{t,\alpha}$ and $\gE_{t,\alpha}^c$ separately.
     \subsubsection{The case of $\gE_{t,\alpha}^c$}\label{app: Etc case for jt}
        Even on $\gE_{t,\alpha}^c$, the results in Lemma~\ref{lem: stab and pen} are still valid for $j_t$, which implies that
        \begin{equation*}
            \E\qty[\hat{\ell}_{t,j_t}\qty(\phi_{j_t}(\eta_t \hat{L}_t) - \phi_{j_t}(\eta_t \hat{L}_{t+1}))\middle| \hat{L}_t] \leq \frac{e\alpha}{\beta_t} p_{t,j_t}^{1/\alpha} = \frac{e\alpha}{\beta_t},
        \end{equation*}
        where the last equality holds since $\sigma_{t,j_t}=1$ and $\hat{\uL}_{t,j_t}=0$ must hold by definition.
        Therefore, it suffices to show that $1\leq a\sum_{i \ne j_t} p_{t,i}^{1/\alpha}$ for some $K$-independent constant $a$, where we show the results for the case of $a=2$.
        Note that $\sum_{i\ne j_t} 1/(1+\eta_t\hat{\uL}_{t,i})^\alpha \geq \xi_\alpha \geq 1/2$ by definition of $\gE_{t,\alpha}^c$.
        
        Firstly, assume for all $i\ne j_t$ that
        \begin{equation*}
            \frac{1}{(1+\eta_t \hat{\uL}_{t,i})^\alpha} \leq \frac{1}{\sigma_{t,i}} \iff p_{t,i} = \frac{1}{(1+\eta_{t}\hat{\uL}_{t,i})^\alpha}.
        \end{equation*}
        Then by definition of $\gE_{t,\alpha}^c$, we obtain that
        \begin{align*}
           \frac{1}{2} \leq \sum_{i \ne j_t} \frac{1}{(1+\eta_t\hat{\uL}_{t,i})^\alpha} \leq \sum_{i \ne j_t} \frac{1}{1+\eta_t\hat{\uL}_{t,i}} = \sum_{i \ne j_t} p_{t,i}^{\frac{1}{\alpha}}.
        \end{align*}
        Next, let us consider the case where there exists an arm $i \ne j_t$ such that $1/(1+\eta_t\hat{\uL}_{t,i})^\alpha \geq 1/\sigma_{t,i}$.
        In this case, since $1/(1+z)^\alpha$ is decreasing with respect to $z$, arms $j$ with $\sigma_{t,j} \leq \sigma_{t,i}$ should satisfy that
        \begin{equation*}
            \frac{1}{(1+\eta_t \hat{\uL}_{t,j})^\alpha} \geq \frac{1}{(1+\eta_t \hat{\uL}_{t,i})^\alpha} \geq \frac{1}{\sigma_{t,i}}, \, \forall j \text{ s.t. }\sigma_{t,j} \leq \sigma_{t,i}.
        \end{equation*}
        This implies that
        \begin{align*}
            \sum_{i\ne j_t} p_{t,i}^{\frac{1}{\alpha}} \geq  \sum_{i\ne j_t} p_{t,i} \geq \sum_{j : \sigma_{t,j} \leq \sigma_{t,i}, j\ne j_t} \frac{1}{\sigma_{t,i}} = 1-\frac{1}{\sigma_{t,i}} \geq \frac{1}{2}. 
        \end{align*}
        Therefore, for any cases, we obtain that
        \begin{equation*}
            \E\qty[\I[\gE_{t,\alpha}]\hat{\ell}_{t,j_t}\qty(\phi_{j_t}(\eta_t \hat{L}_t) - \phi_{j_t}(\eta_t \hat{L}_{t+1}))\middle| \hat{L}_t] \leq \I[\gE_{t,\alpha}]\frac{e\alpha}{\beta_t} \leq \I[\gE_{t,\alpha}]\frac{2e\alpha}{\beta_t} \sum_{i\ne j_t}p_{t,i}^{\frac{1}{\alpha}},
        \end{equation*}
        which concludes the proof for the case on $\gE_{t,\alpha}^c$.
        
     \subsubsection{The case of $\gE_{t,\alpha}$}
        On $\gE_{t,\alpha}$, it is clear that $\eta_t \hat{\uL}_{t,i} \geq \xi_\alpha^{-1/\alpha}-1>0$ holds.
        Let $\zeta_\alpha$ be an $\alpha$-dependent dependent constant in $(0, \xi_\alpha^{-1/\alpha}-1)$, specified later.
        Then, whenever $\hat{\ell}_{t,j_t} \leq \zeta_\alpha/\eta_t$, we can apply the same techniques in Lemma 25 of \citet{lee2024follow}, which shows that
        \begin{align*}
            \E&\qty[\I[\gE_{t,\alpha},\hat{\ell}_{t,j_t}\leq \zeta_\alpha \beta_t]\hat{\ell}_{t,j_t}\qty(\phi_{j_t}(\eta_t \hat{L}_t) - \phi_{j_t}(\eta_t(\hat{L}_t + \hat{\ell}_{t,j_t}e_{j_t}))) \middle| \hat{L}_t] \tag{$\beta_t = 1/\eta_t$}\\
            &\leq \E\qty[\I[\gE_{t,\alpha},\hat{\ell}_{t,j_t}\leq \zeta_\alpha \beta_t] e^2(1-e^{-1})\hat{\ell}_{t,j_t}^2 \sum_{i \ne j_t} \frac{\eta_t \alpha}{(1+\eta_t(\hat{\uL}_{t,i}-\zeta_\alpha))^{\alpha+1}} \middle| \hat{L}_t] \\
            &\leq \E\qty[\I[\gE_{t,\alpha},\hat{\ell}_{t,j_t}\leq \zeta_\alpha \beta_t] e^2(1-e^{-1}) \frac{2\ell_{t,j_t}^2 \I[i_t = j_t]}{w_{t,j_t}^2}\sum_{i \ne j_t} \frac{\eta_t \alpha}{(1+\eta_t(\hat{\uL}_{t,i}-\zeta_\alpha))^{\alpha+1}} \middle| \hat{L}_t] \\
            &= \E\qty[\I[\gE_{t,\alpha},\hat{\ell}_{t,j_t}\leq \zeta_\alpha \beta_t] \frac{2e^2(1-e^{-1})\ell_{t,j_t}^2}{w_{t,j_t}}\sum_{i \ne j_t} \frac{\eta_t \alpha}{(1+\eta_t(\hat{\uL}_{t,i}-\zeta_\alpha))^{\alpha+1}} \middle| \hat{L}_t] \\
            &\leq \frac{2e^2(1-e^{-1})}{1-\xi_\alpha} \sum_{i \ne j_t} \frac{\eta_t \alpha}{(1+\eta_t(\hat{\uL}_{t,i}-\zeta_\alpha))^{\alpha+1}} , \numberthis{\label{eq: Et decomposition all alpha}}
        \end{align*}
        where the last inequality follows from (\ref{eq: wjt lower bound on E}) and $\ell_{t} \in [0,1]^K$.
        \paragraph{When $\alpha \in (1,2)$.}
        In this case, we set $\xi_\alpha =1/2$, where $\xi_\alpha^{-1/\alpha}-1 \in (\sqrt{2}-1,1)$.
        Here, we take $\zeta_\alpha = (\xi_\alpha^{-1/\alpha}-1)/2$ for analytical simplicity, which implies $\hat{\ell}_{t,j_t} \leq \hat{\uL}_{t,i}/2.$
        Then, (\ref{eq: Et decomposition all alpha}) satisfies that
        \begin{align*}
            \frac{2e^2(1-e^{-1})}{1-\xi_\alpha} \sum_{i \ne j_t} \frac{\eta_t \alpha}{(1+\eta_t(\hat{\uL}_{t,i}-\zeta_\alpha))^{\alpha+1}} &\leq \sum_{i\ne j_t} \eta_t\frac{2^{\alpha+2}e^2(1-e^{-1})\alpha}{(2+\eta_t \hat{\uL}_{t,i})^{\alpha+1}} \\
            &\leq \sum_{i\ne j_t} \eta_t\frac{2^{\alpha+2}e^2(1-e^{-1})\alpha}{(1+\eta_t\hat{\uL}_{t,i})^{\alpha+1}}. \numberthis{\label{eq: Et bound for small alpha}}
        \end{align*}
        Since $\alpha \in (1,2)$, the multiplicative constant is at most $16e^2(1-e^{-1}) \lesssim 75$.
        Then, on $\gE_{t,\alpha}$ and $\hat{\ell}_{t,j_t} \leq \zeta_\alpha \beta_t$, it remains to show how (\ref{eq: Et bound for small alpha}) provides the desired results, i.e., the upper bounds in terms of $\sum_{i \ne j_t} p_{t,i}^{1/\alpha}$.

        By definition of $\gE_{t,\alpha}$, as mentioned above, $\hat{\uL}_{t,i} >0$ holds for any $i \ne j_t$ on $\gE_{t,\alpha}$.
        Since $1/(1+x)^\alpha$ is decreasing with respect to $x$, for any $i \ne j_t$, it holds that
        \begin{align*}
            \sum_{i\ne j_t} \frac{1}{(1+\eta_t \hat{\uL}_{t,i})^\alpha}  \geq \frac{\sigma_{t,i}-1}{(1+\eta_t\hat{\uL}_{t,i})^{\alpha}},
        \end{align*}
        which implies that on $\gE_{t,\alpha}$ for $\alpha \in (0,1)$
        \begin{equation}\label{eq: jt decomp last qt and alpha}
            \frac{1}{(1+\eta_t\hat{\uL}_{t,i})^\alpha} \leq \frac{1}{2} \frac{1}{\sigma_{t,i}-1} \leq \frac{1}{\sigma_{t,i}} \implies p_{t,i} = \frac{1}{(1+\eta_t\hat{\uL}_{t,i})^\alpha},
        \end{equation}
        where the last inequality holds since $\sigma_{t,i}\geq 2$ for $i \ne j_t$ by definition.
        Therefore, we obtain that
        \begin{align*}
            &\E\qty[\I[\gE_{t,\alpha},\hat{\ell}_{t,j_t}\leq \zeta_\alpha \beta_t]\hat{\ell}_{t,j_t}\qty(\phi_{j_t}(\eta_t \hat{L}_t) - \phi_{j_t}(\eta_t(\hat{L}_t + \hat{\ell}_{t,j_t}e_{j_t}))) \middle| \hat{L}_t] \\
            &\hspace{15em}\leq \sum_{i\ne j_t} \eta_t \frac{2^{\alpha+2}e^2(1-e^{-1})\alpha}{\sigma_{t,i}} \frac{1}{1+\eta_t\hat{\uL}_{t,i}} \\
            &\hspace{15em}\leq\sum_{i\ne j_t} 2^{\alpha+1}e^2(1-e^{-1})\alpha  \frac{\eta_t}{1+\eta_t\hat{\uL}_{t,i}} \\
            &\hspace{15em}= \sum_{i\ne j_t} 2^{\alpha+1}e^2(1-e^{-1})\alpha  \frac{p_{t,i}^{1/\alpha}}{\beta_t} \tag{$\beta_t = 1/\eta_t$},
        \end{align*}
        as desired.
        Finally, it remains to consider the case $\hat{\ell}_{t,j_t} > \zeta_\alpha \beta_t$ on $\gE_{t,\alpha}$.
        We show that this event cannot occur under the design of Algorithm~\ref{alg: FTPL MAB}, which employs CGR II-biased with the number of maximum resampling steps $G_t$.
        Therefore, it suffices to show that when $\1[\gE_{t,\alpha}, i_t=j_t]=1$,
        \begin{align*}
           \hat{\ell}_{t,j_t} \leq \ell_{t,j_t} G_t \leq G_t \leq \zeta_\alpha \beta_t.
        \end{align*}
        For $\alpha \in (1,2)$, we set $G_t = 2\log t$ when both $\gE_{t,\alpha}$ and $i_t = j_t$ occur.
        Therefore, Lemma~\ref{lem: lr lb in MAB} concludes the proof.

        \paragraph{When $\alpha \geq 2$.}
        Let $\zeta_\alpha = 1-2^{-1/(\alpha+1)}$, such that $\frac{1}{(1-\zeta_\alpha)^{\alpha+1}}= 2$.
        Then, on $\gE_{t,\alpha}$, it is clear that $\eta_t\hat{\uL}_{t,i} \geq \zeta_\alpha$ for all $i \ne j_t$ by the choice of $\xi_\alpha$ and $\zeta_\alpha$. 
        Whenever $\hat{\ell}_{t,j_t} \leq \zeta_\alpha/\eta_t$, we can apply the same techniques as the case of $\alpha \in (1,2)$, which shows that
        \begin{align*}
            \E&\qty[\I[\gE_{t,\alpha},\hat{\ell}_{t,j_t}\leq \zeta_\alpha \beta_t]\hat{\ell}_{t,j_t}\qty(\phi_{j_t}(\eta_t \hat{L}_t) - \phi_{j_t}(\eta_t(\hat{L}_t + \hat{\ell}_{t,j_t}e_{j_t}))) \middle| \hat{L}_t] \\
            &\leq \E\qty[\I[\gE_{t,\alpha},\hat{\ell}_{t,j_t}\leq \zeta_\alpha \beta_t] \frac{ e^2(1-e^{-1})}{(1-\zeta_\alpha)^{\alpha+1}} \hat{\ell}_{t,j_t}^2 \sum_{i \ne j_t} \frac{\eta_t \alpha}{(1+\eta_t\hat{\uL}_{t,i})^{\alpha+1}} \middle| \hat{L}_t] \\
            &\leq \E\qty[\I[\gE_{t,\alpha},\hat{\ell}_{t,j_t}\leq \zeta_\alpha \beta_t] \frac{ e^2(1-e^{-1})}{(1-\zeta_\alpha)^{\alpha+1}} \frac{2\ell_{t,j_t}^2 \I[i_t = j_t]}{w_{t,j_t}^2}\sum_{i \ne j_t} \frac{\eta_t \alpha}{(1+\eta_t\hat{\uL}_{t,i})^{\alpha+1}} \middle| \hat{L}_t] \\
            &= \E\qty[\I[\gE_{t,\alpha},\hat{\ell}_{t,j_t}\leq \zeta_\alpha \beta_t] \frac{4e^2(1-e^{-1})\ell_{t,j_t}^2}{w_{t,j_t}}\sum_{i \ne j_t} \frac{\eta_t \alpha}{(1+\eta_t\hat{\uL}_{t,i})^{\alpha+1}} \middle| \hat{L}_t] \\
            &\leq \frac{4e^2(1-e^{-1})}{1-\xi_\alpha}  \sum_{i \ne j_t} \frac{\eta_t \alpha}{(1+\eta_t\hat{\uL}_{t,i})^{\alpha+1}} \tag{by (\ref{eq: wjt lower bound on E}) and $\ell_{t} \in [0,1]^K$} \\
            &\leq 13e^2(1-e^{-1})\alpha \sum_{i \ne j_t} \frac{\eta_t}{(1+\eta_t\hat{\uL}_{t,i})^{\alpha+1}},
        \end{align*}
        where the last inequality follows from that $\frac{1}{\xi_\alpha-1}=\frac{(1+\zeta_\alpha)^\alpha}{(1+\zeta_\alpha)^\alpha-1}$ is decreasing with respect to $\alpha>1$ and its value at $\alpha=2$ is less than $3.2$.

        Similarly, it remains to consider the case $\hat{\ell}_{t,j_t} > \zeta_\alpha \beta_t$ on $\gE_{t,\alpha}$ for $\alpha \geq 2$ case.
        As one can easily expect, since $\beta_t = \Omega(\sqrt{t})$, it is possible to directly apply Lemma~\ref{lem: honda}, which provides an additional term whose summation over $t$ is at most $\gO(K^{2/\alpha-1}) \leq \gO(1)$ for $\alpha \geq 2$.
        Since this direct application requires to modify the constant term appear on $\gE_{t,\alpha}^c$ case, for the coherence with $\alpha \in (1,2)$, we show that
        \begin{align*}
            \hat{\ell}_{t,j_t} \leq \ell_{t,j_t} G_t \leq K\log t \leq \zeta_\alpha \sqrt{\beta_1^2 + 2(t-1)}.
        \end{align*}
        for $t \geq t_0(\alpha,K)$.
        Since $\beta_1^2 = \gO(\alpha^2 K^{1-2/\alpha})$ by our choice, it suffices to find the solution of $\log t \leq a\sqrt{t}$ for $a=\zeta_\alpha\sqrt{2}/K$.
        Let $t=y^2$.
        Then,
        \begin{align*}
            \log y = ay/2 \iff y =e^{\frac{a}{2}y} \iff ze^z = -\frac{a}{2}. \tag{$z= -\frac{a}{2}y$}
        \end{align*}
        Since $-\frac{a}{2} = -\frac{\zeta_\alpha \sqrt{2}}{K} \in (-1/e, 0)$, the above equality admits two real solution $z= W_{n}(-a/2)$, where $W_n(\cdot)$ denotes the Lambert W function with branch $n$ and $n \in \{-1,0\}$ (i.e., only the principal branch since we consider the real value)~\citep[Section 4]{olver2010nist}.
        Therefore, we obtain the desired results for any $t \geq t_0:=\frac{ K^2}{2\zeta_\alpha^2}\qty(W_{-1}\qty(-\frac{\zeta_\alpha \sqrt{2}}{K}))^2$.
        For sufficiently small $a$, we can approximate the Lambert W function as~\citep[4.13.11]{olver2010nist}
        \begin{equation*}
            -W_{-1}(-a) \approx \log \frac{1}{a} + \log\log \frac{1}{a} \implies t_0(\alpha,K) \approx \qty(\frac{K}{\sqrt{2} \zeta_\alpha})^2 \qty(\log \frac{K}{\sqrt{2}\zeta_\alpha}+\log\log \frac{K}{\sqrt{2} \zeta_\alpha})^2,
        \end{equation*}
        which implies that $t_0(\alpha, K) \leq \gO(\alpha^2 K^2\log^2 (\alpha K))$ with the current choice of $\zeta_\alpha$.

        In sum, for $t \geq t_{0}(\alpha,K)$, we obtain
        \begin{align*}
            \E\qty[\I[\gE_{t,\alpha}]\hat{\ell}_{t,j_t}\qty(\phi_{j_t}(\eta_t \hat{L}_t) - \phi_{j_t}(\eta_t \hat{L}_{t+1}))\middle| \hat{L}_t] \leq  \sum_{i\ne j_t} \gO\qty(\frac{\eta_t\I[\gE_{t,\alpha}]}{(1+\eta_t \hat{\uL}_{t,i})^{\alpha+1}}).
        \end{align*}
        Similarly to the case of $\alpha\in (1,2)$ in (\ref{eq: jt decomp last qt and alpha}), by definition, we have for $i \ne j_t$
        \begin{align*}
            \frac{1}{(1+\eta_t \hat{\uL}_{t,i})^{\alpha}}\leq \frac{\xi_\alpha}{\sigma_{t,i}-1} \leq 2\xi_\alpha \frac{1}{\sigma_{t,i}}, \tag{$\because \sigma_{t,i}\geq 2, \, \forall i \ne j_t$}
        \end{align*}
        which concludes the proof.
\end{proof}

\begin{remark}
    In the current analysis, we obtain a loose bound in constant, especially the bound on $\gE_{t,\alpha}$, e.g. $75$ in $\alpha\in(1,2)$ case.
    However, we can tune both $\xi_\alpha$ and $\zeta_\alpha$ to reduce the constant term.
    In $\alpha \in (1,2)$ example, we can choose $\zeta_\alpha = 4(\xi_{\alpha}^{-1/\alpha}-1)/5$, we change the term $2^{\alpha+1}$ to $(5/4)^{\alpha+1}$, which results in around $18$ (and thus $9$ in the final bound) instead of $75$ in the multiplicative constant.
    Note that when we modify the choice of $\zeta_\alpha$, this will affect the bound on $\gE_{t,\alpha}^c$, where multiplicative constant $2$ becomes $1/\xi_\alpha$ and we need to modify the constant in $1/t$ term in $\beta_t$ in (\ref{def: lr in MAB}) since $4/(2^{1/\alpha}-1)$ is chosen to satisfy $G_t \approx \log t/(1-\xi_\alpha) \leq \zeta_\alpha \beta_t$.
    Therefore, the choice of $\xi_\alpha$ and $\zeta_\alpha$ should take multiple factors into account simultaneously.
\end{remark}

\subsection{Proof of Lemma~\ref{lem: spm regret}}
Although the proof of Lemma~\ref{lem: spm regret} can be directly obtained by Lemmas 9 and 10 in \citet{pmlr-v247-ito24a}, we provide the results for the completeness since our $\beta_{t}$ in (\ref{def: lr in MAB}) for $\alpha \in (1,2)$ includes additional term and for $\alpha \geq 2$ includes the restriction on the exponential growth.
Before the proof, we introduce the following results.
\begin{lemma}[Lemma 10 of \citet{pmlr-v247-ito24a}]\label{lem: ito 10}
It holds that
\begin{equation*}
    \sum_{t=1}^T \frac{z_t}{\sqrt{\sum_{s=1}^t \frac{z_s}{h_s}}} \leq \gO\qty(\min\qty{\sqrt{\log T \sum_{t=1}^T h_tz_t }+ \sqrt{h_{\max} z_{\max}}, \sqrt{h_{\max}\sum_{t=1}^T z_t}}).
\end{equation*}
\end{lemma}

\begin{proof}
\textbf{of Lemma 3 }
    Define an auxiliary sequence $\beta_t' = \sqrt{\beta_1^2 + 2\sum_{s=1}^{t-1}z_s/h_s}$ so that $\beta_t' \leq \beta_t$ holds by their definitions for $\alpha \in (1,2)$.
    Note that for the case of $\alpha \in (1,2)$, since we take the maximum of two values, $\beta_t'\leq \beta_t$ always holds.
    Then, it remains to follow the proofs in previous results~\citep{pmlr-v247-ito24a, nguyen2025data}.

    Define a set of rounds $\gT := \qty{t\in [T]: \beta_{t+1}' \geq \sqrt{2}\beta_t'}$, where we can rewrite 
    \begin{equation*}
        \sum_{t=1}^T \frac{z_t}{\beta_t} = \sum_{t\in \gT} \frac{z_t}{\beta_t} + \sum_{t \in\gT^c} \frac{z_t}{\beta_t}.
    \end{equation*}
    By definition of $\gT$ and $\beta_1 = \gO(\alpha K^{\frac{1}{2}-\frac{1}{\alpha}})$, we have
    \begin{align*}
        \sum_{t \in \gT} \frac{z_t}{\beta_t} \leq \sum_{t\in \gT} \frac{z_{\max}}{\beta_t'} &\leq \sum_{s=0}^\infty \qty(\frac{1}{\sqrt{2}})^s \frac{z_{\max}}{\beta_1} \\
        &\leq (2+\sqrt{2}) \frac{z_{\max}}{\beta_1} \\
        &\leq (2+\sqrt{2}) \frac{\alpha \sum_{n=2}^K n^{-1/\alpha}}{\beta_1}\\
        &\leq  \frac{\alpha^2(2+\sqrt{2})}{\alpha-1} \frac{(K+1)^{1-1/\alpha}-1}{\beta_1 } \\
        &\leq  \frac{\alpha^2(2+\sqrt{2})}{\alpha-1} \frac{K^{1-1/\alpha}}{\beta_1} = \gO\qty(\frac{\alpha}{\alpha-1}\sqrt{K}). \numberthis{\label{eq: lem3 eq1}}
    \end{align*}
    On the other hand, we have
    \begin{align*}
        \sum_{t \in \gT^c} \frac{z_t}{\beta_t} \leq \sum_{t \in \gT^c} \frac{z_t}{\beta_t'} \leq \sqrt{2} \sum_{t \in \gT^c} \frac{z_t}{\beta_{t+1}'} = \sqrt{2}\sum_{t \in \gT^c} \frac{z_t}{\sqrt{\beta_1^2 + 2\sum_{s=1}^{t}z_s/h_s}} \leq \sum_{t=1}^T \frac{z_t}{\sqrt{\sum_{s=1}^t \frac{z_s}{h_s}}},
    \end{align*}
    which concludes the proof for $\alpha \in (1,2)$ by applying Lemma~\ref{lem: ito 10}.

    For $\alpha \geq 2$, define a set of rounds $\gT := \{t \in [T]: \beta_{t+1} \geq 2^{1/\alpha}\beta_{t} \}$, where $\beta_{t+1} = 2^{1/\alpha}\beta_t$ by definition of $\beta_t$ in (\ref{def: lr in MAB}) for $\alpha \geq 2$.
    Note that for $t \in \gT$, it holds that $\frac{z_t}{\beta_t h_t} \geq \beta_t(2^{1/\alpha}-1)$. 
    Then, by following the same steps in (\ref{eq: lem3 eq1}) with $2^{1/\alpha}$ instead of $\sqrt{2}$, we obtain
    \begin{align*}
        \sum_{t\in \gT} \frac{z_t}{\beta_t} \leq \frac{\alpha}{(\alpha-1)(2^{1/\alpha}-1)} K^{\frac{1}{2}-\frac{1}{\alpha}} \leq \frac{\alpha^2}{(\alpha-1)\log2} K^{\frac{1}{2}-\frac{1}{\alpha}}.
    \end{align*}
    On the other hand, since $\beta_{t+1} = \beta_t + \frac{z_t}{\beta_t h_t}$ holds on $\gT^c$, we have
    \begin{align*}
        \sum_{t\in \gT^c} \frac{z_t}{\beta_t} &\leq \sum_{t \in \gT^c } \frac{2^{1/\alpha} z_t}{\beta_{t+1}} \\
        &\leq \sum_{t \in \gT^c} \frac{2^{1/\alpha} z_t}{\sqrt{\beta_1^2 + 2\sum_{s \in \gT^c \cap [t]}\frac{z_s}{h_s}}} \\
        &\leq  \sum_{t \in \gT^c} \frac{z_t}{\sqrt{\sum_{s \in \gT^c \cap [t]}\frac{z_s}{h_s}}},
    \end{align*}
    where we only consider the effect of updates in $\beta_t$ on $\gT^c$.
    By applying Lemma~\ref{lem: ito 10}, we obtain
    \begin{align*}
          \sum_{t\in \gT^c} \frac{z_t}{\beta_t}& \leq \gO\qty(\min\qty{\sqrt{\log |\gT^c| \sum_{t\in \gT^c} h_tz_t }+ \sqrt{h_{\max} z_{\max}}, \sqrt{h_{\max}\sum_{t\in \gT^c} z_t}}) \\
          &\leq \gO\qty(\min\qty{\sqrt{\log T \sum_{t=1}^T h_tz_t }+ \sqrt{h_{\max} z_{\max}}, \sqrt{h_{\max}\sum_{t=1}^T z_t}}),
    \end{align*}
    which concludes the proof.
\end{proof}

\subsection{Proof of Lemma~\ref{lem: lr lb in MAB}}\label{app: lr lb}
\begin{proof}
    For $\alpha \in (1,2)$, it is clear that
    \begin{align*}
        \beta_{t} \geq \beta_1 + \frac{4}{2^{1/\alpha}-1} \sum_{s=1}^{t-1} \frac{1}{s} \geq \beta_1 + \frac{4}{2^{1/\alpha}-1} \log t.
    \end{align*}
    For $\alpha \geq 2$, by definition of $z_s$ and $h_s$ in (\ref{def: hz in MAB}), we have for $\alpha \geq 2$
    \begin{align*}
        \frac{z_s}{h_s} = \frac{\alpha \sum_{i\ne j_s} q_{s,i}^{1/\alpha}}{{\frac{\alpha}{\alpha-1}\sum_{i\ne j_s} q_{s,i}^{1-1/\alpha}}} = (\alpha-1) \frac{\sum_{i\ne j_s} q_{s,i}^{1/\alpha}}{{\sum_{i\ne j_s} q_{s,i}^{1-1/\alpha}}} \geq  1.
    \end{align*}
     Since $q_{t,i} \in (0,1)$ for $i\ne j_t$ by the choice of $q_{t,i}$ in (\ref{def: surrogate}) and $\frac{1}{\alpha} \leq 1-\frac{1}{\alpha}$, i.e., $q_{t,i}^{1/\alpha} \geq q_{t,i}^{1-1/\alpha}$ always holds.   
    Note that whenever $\beta_{t+1} = 2^{1/\alpha} \beta_t$ occurs, this means $\beta_{t+1} = \beta_t + (2^{1/\alpha}-1)\beta_t$.
    Since $\beta_1 = 2\alpha K^{\frac{1}{2}-\frac{1}{\alpha}}$ and $2^{1/\alpha}-1 \geq \frac{\log 2}{\alpha}$, the increment is always larger than $2\log 2$.
    Therefore, it holds that $\beta_{t+1} \geq \beta_t + \frac{1}{\beta_t}$ for any $t\in \sN$.
    This implies that $\beta_t \geq \sqrt{\beta_1^2 + 2\sum_{s=1}^{t-1}1} \geq \sqrt{2t}$.
\end{proof}

\section{Proof of Theorem~\ref{thm: bobw mab}}
In this section, we prove the BOBW guarantee of Algorithm~\ref{alg: FTPL MAB}.
\subsection{Adversarial regime}
In this regime, it suffices to show that $h_{\max}\sum_{t=1}^T z_t$ is at most $KT$ from the second term in (\ref{lem: ito 10}).
By definition, for any $t$, it holds that
\begin{align*}
    z_t =\alpha \sum_{i\ne j_{t+1}} q_{t+1,i}^{1/\alpha} &\leq \sum_{i\ne j_{t+1}} \frac{\alpha}{\sigma_{t+1,i}^{1/\alpha}} \\
    &= \sum_{n=2}^K \frac{\alpha}{n^{1/\alpha}} \leq \frac{\alpha^2}{\alpha-1} ((K+1)^{1-1/\alpha}-1) \leq  \frac{\alpha^2}{\alpha-1} K^{1-1/\alpha}, \numberthis{\label{eq: z max mab}}
\end{align*}
and
\begin{align*}
    h_t = \frac{\alpha}{\alpha-1}\sum_{i\ne j_{t+1}} q_{t+1,i}^{1-1/\alpha} &\leq \frac{\alpha}{\alpha-1}\sum_{i\ne j_{t+1}} \frac{1}{\sigma_{t+1,i}^{1-1/\alpha}} \\
    &= \frac{\alpha}{\alpha-1}\sum_{n=2}^K \frac{1}{n^{1-1/\alpha}} \leq \frac{\alpha^2((K+1)^{1/\alpha}-1)}{\alpha-1} \leq \frac{\alpha^2}{\alpha-1}K^{1/\alpha}. \numberthis{\label{eq: h max mab}}
\end{align*}
Therefore,
\begin{align*}
    h_{\max}\sum_{t=1}^T z_t \leq \sum_{t=1}^T \qty(\frac{\alpha^2}{\alpha-1})^2 K = \qty(\frac{\alpha^2}{\alpha-1})^2 KT,
\end{align*}
which concludes the proof for the adversarial regime.

\subsection{Adversarial regime with self-bounding constraint}\label{app: self-bounding mab}
To analyze the regret in this regime, we introduce the event $\gD_{t,\alpha}$ defined by
\begin{align*}
    \gD_{t,\alpha} := \qty{\sum_{i\ne j_t} \frac{1}{(2^{1/\alpha}+\eta_t\hat{\uL}_{t,i})^\alpha} \leq \frac{1}{2}},
\end{align*}
which is a slightly modified version of $\gE_{t,\alpha}$.
This definition is to utilize the previous results in \citet{kim2025follow}, while there is a subtle difference due to $j_t$ parts instead of $i^*$.
On this event, we have the following lemma, which slightly improves Lemma 10 in \citet{kim2025follow} in the constant factor.
\begin{lemma}\label{lem: improved q and w}
    For Algorithm~\ref{alg: FTPL MAB} with $\alpha >1$, it holds that on $\gD_{t,\alpha}$
    \begin{align*}
        w_{t,i} \geq \frac{1}{8} \frac{1}{(1+\eta_t\hat{\uL}_{t,i})^\alpha}, \, \forall i \ne j_t.
    \end{align*}
    In addition, the current best arm satisfies $\frac{1}{4} \leq w_{t,j_t}$ on $\gD_{t,\alpha}$. 
\end{lemma}
While \citet{kim2025follow} considered the case $j_t=i^*$, the following results can be directly obtained.
\begin{lemma}[Lemma 11 in \citet{kim2025follow}]\label{lem: lem 11 in kim}
    On $\gD_{t,\alpha}^c$, $w_{t,j_t}\leq \frac{1+e^{-1/2}}{2}$.
\end{lemma}
The implication of $\gD_{t,\alpha}$ is that we can link the $q_{t,i}$ and $w_{t,i}$ on this event, which will be not rare when the self-bounding constraint is small, i.e., as close as to stochastic settings.
Specifically, we have
\begin{align*}
    w_{t,i}\geq \frac{1}{8} \frac{1}{(1+\eta_t \hat{\uL}_{t,i})^\alpha} \geq \frac{p_{t,i}}{8} \implies w_{t,i} \in \qty[\frac{p_{t,i}}{8}, p_{t,i}], \text{ and } p_{t,i} \leq 8w_{t,i}.
\end{align*}
Then, by Lemma~\ref{lem: formal pt and qt}, we have
\begin{equation}\label{eq: w q bound in mab}
    q_{t,i} \leq 16 w_{t,i}.
\end{equation}
Then, similarly to the recent BOBW analysis of FTPL~\citep{pmlr-v201-honda23a, lee2024follow}, we consider the analysis on $\gD_{t,\alpha}$ and $\gD_{t,\alpha}^c$ separately.
In terms of Lemma~\ref{lem: ito 10}, what we will show is that
\begin{align*}
    \sum_{t=1}^{T-1} \frac{z_t}{\beta_t} \lesssim \sqrt{\log T \qty(\sum_{t=1}^{T-1} \I[\gD_{t+1,\alpha}]h_t z_t + \sum_{t=1}^{T-1} \I[\gD_{t+1,\alpha}^c]h_t z_t)} + \gO\qty(\frac{\alpha^2}{\alpha-1}\sqrt{K}).
\end{align*}
Note that the last term is directly obtained by the results (\ref{eq: z max mab}) and (\ref{eq: h max mab}) in adversarial regime, where we showed that 
\begin{equation*}
    \sqrt{h_{\max}z_{\max}} \leq \frac{\alpha^2}{\alpha-1}\sqrt{K}.
\end{equation*}

\paragraph{On $\gD_{t,\alpha}$.}
Note that $w_{t,j_t} \geq w_{t,i}$ for any $i \in [K]$ and $t\in \sN$.
By H\"older's inequality, we have
\begin{align*}
   \I[\gD_{t,\alpha}] z_{t-1} =  \I[\gD_{t,\alpha}] \alpha \sum_{i\ne j_t} q_{t,i}^{1/\alpha} &\leq  \I[\gD_{t,\alpha}] \alpha  16^{1/\alpha} \sum_{i\ne j_t} w_{t,i}^{1/\alpha} \tag{by (\ref{eq: w q bound in mab})}\\ 
   &\leq \I[\gD_{t,\alpha}] \alpha  16^{1/\alpha} \sum_{i\ne i^*} w_{t,i}^{1/\alpha} \tag{$w_{t,j_t}\geq w_{t,i}, \forall i$}\\
   &= \I[\gD_{t,\alpha}] \alpha 16^{1/\alpha} \sum_{i\ne i^*} \frac{1}{\Delta_i^{1/\alpha}}(\Delta_i w_{t,i})^{1/\alpha}  \\
   &\leq \I[\gD_{t,\alpha}] \alpha  16^{1/\alpha} \qty(\sum_{i\ne i^*} \frac{1}{\Delta_i^{1/(\alpha-1)}})^{1-\frac{1}{\alpha}}\qty(\sum_{i\ne i^*}\Delta_i w_{t,i})^{1/\alpha} 
\end{align*}
and
\begin{align*}
    \I[\gD_{t,\alpha}] h_{t-1} &=  \I[\gD_{t,\alpha}] \frac{\alpha}{\alpha-1} \sum_{i\ne j_t} q_{t,i}^{1-1/\alpha} \\
    &\leq  \I[\gD_{t,\alpha}] \frac{\alpha}{\alpha-1} 16^{1-1/\alpha} \sum_{i\ne j_t} w_{t,i}^{1-1/\alpha} \\
    &\leq \I[\gD_{t,\alpha}] \frac{\alpha}{\alpha-1}16^{1-1/\alpha} \sum_{i\ne i^*} w_{t,i}^{1-1/\alpha} \\ 
   &= \I[\gD_{t,\alpha}] \frac{\alpha}{\alpha-1}16^{1-1/\alpha} \sum_{i\ne i^*} \frac{1}{\Delta_i^{1-1/\alpha}}(\Delta_i w_{t,i})^{1-1/\alpha}  \\
   &\leq \I[\gD_{t,\alpha}] \frac{\alpha}{\alpha-1}16^{1-1/\alpha} \qty(\sum_{i\ne i^*} \frac{1}{\Delta_i^{\alpha-1}})^{\frac{1}{\alpha}}\qty(\sum_{i\ne i^*}\Delta_i w_{t,i})^{1-1/\alpha} .  \numberthis{\label{eq: ht wt mab}}
\end{align*}
Therefore,
\begin{align*}
    \I[\gD_{t+1,\alpha}] h_t z_t &\leq \I[\gD_{t+1,\alpha}] \omega'(\Delta)\inp{\Delta}{w_{t+1}},
\end{align*}
where
\begin{align*}
    \omega'(\Delta) =  \frac{16\alpha^2}{\alpha-1} \qty(\sum_{i \ne i^*}\Delta_i^{-\frac{1}{\alpha-1}})^{1-\frac{1}{\alpha}} \qty(\sum_{i \ne i^*} \Delta_i^{1-\alpha})^{\frac{1}{\alpha}}. 
\end{align*}
\paragraph{On $\gD_{t,\alpha}^c$.}
In this case, we have
\begin{align*}
    \I[\gD_{t,\alpha}^c]h_{t-1} z_{t-1} \leq  \I[\gD_{t,\alpha}^c] \frac{\alpha^2}{\alpha-1} \sum_{i\ne j_t} q_{t,i}^{1/\alpha} \cdot \sum_{i\ne j_t} q_{t,i}^{1-1/\alpha} &\leq \I[\gD_{t,\alpha}^c] \frac{\alpha^2}{\alpha-1} \frac{\alpha}{\alpha-1}K^{1-1/\alpha} \cdot  \alpha K^{1/\alpha} \\
    &=  \I[\gD_{t,\alpha}^c]  \qty(\frac{\alpha^2}{\alpha-1})^2 K.
\end{align*}
\paragraph{Derivation of the desired results.} 
Therefore, we obtain that
\begin{align*}
    \sum_{t=1}^{T-1} \frac{z_t}{\beta_t} &\lesssim \sqrt{\log T \qty(\sum_{t=1}^{T} \I[\gD_{t,\alpha}]\omega'(\Delta)\inp{\Delta}{w_{t}} + \sum_{t=1}^{T} \I[\gD_{t,\alpha}^c]\qty(\frac{\alpha^2}{\alpha-1})^2 K)} \\
    &\hspace{1em}+ \gO\qty(\frac{\alpha^2}{\alpha-1}\sqrt{K}).
\end{align*}
Combined with the regret upper bounds in Lemmas~\ref{lem: regret decomposition all} and~\ref{lem: spm regret}, we obtain for $\alpha \in (1,2)$ that
\begin{multline}\label{eq: mab regret small alpha}
    \Reg_{\text{FTPL}}(T) \lesssim \gO\qty(\sqrt{\log T\qty(\sum_{t=1}^{T} \omega'(\Delta) \inp{\Delta}{w_t} + \sum_{t=1}^{T}\I[\gD_{t,\alpha}^c]\qty(\frac{\alpha^2}{\alpha-1})^2K) }) \\
    + \gO\qty(\frac{\alpha^2}{\alpha-1} K^{1/\alpha}\log T) + \gO\qty(\frac{\alpha^3\sqrt{K}}{(\alpha-1)^2}).
\end{multline}
and for $\alpha \geq 2$ that
\begin{multline}\label{eq: mab regret large alpha}
     \Reg_{\text{FTPL}}(T) \lesssim \gO\qty(\sqrt{\log T\qty(\sum_{t=1}^{T} \omega'(\Delta) \inp{\Delta}{w_t} + \sum_{t=1}^{T}\I[\gD_{t,\alpha}^c]\qty(\frac{\alpha^2}{\alpha-1})^2 K) }) \\
    + \gO\qty(\frac{\alpha^3}{(\alpha-1)^2} \sqrt{K}) + t_0(\alpha, K).
\end{multline}
Therefore, it remains to control the term related to $\sum_t\I[\gD_{t,\alpha}^c]$.
Recall that, in the adversarial regime with self-bounding constraint $(\Delta,C,T)$, the regret satisfies that
\begin{equation*}
     \Reg(T) \geq \E\qty[\sum_{t=1}^T \inp{\Delta}{w_t}] - C.
\end{equation*}
On $\gD_{t,\alpha}^c$, Lemma~\ref{lem: lem 11 in kim} shows that
\begin{align*}
    \I[\gD_{t,\alpha}^c]\sum_{i=1}^K \Delta_i w_{t,i} =  \I[\gD_{t,\alpha}^c]\sum_{i\ne i^*} \Delta_i w_{t,i} &\geq \I[\gD_{t,\alpha}^c]\Delta_{\min} \sum_{i\ne i^*} w_{t,i} \\
    &\geq \I[\gD_{t,\alpha}^c]\Delta_{\min} \sum_{i\ne j_t} w_{t,i} \tag{$w_{t,j_t}\geq w_{t,i}, \forall i$} \\
    &= \I[\gD_{t,\alpha}^c]\Delta_{\min} (1-w_{t,j_t}) \\
    &\geq \I[\gD_{t,\alpha}^c]\Delta_{\min} \frac{1-e^{-1/2}}{2}.
\end{align*}
This implies that
\begin{equation*}
    \E\qty[\sum_{t=1}^T \I[\gD_{t,\alpha}^c] 0.31 \Delta_{\min} ]\leq \Reg(T) + C \implies \E\qty[\sum_{t=1}^T \I[\gD_{t,\alpha}^c]] \leq \frac{\Reg(T)+C}{0.31 \Delta_{\min}}. 
\end{equation*}
By applying this result into (\ref{eq: mab regret small alpha}) and (\ref{eq: mab regret large alpha}), the regret can be upper bound in the form of
\begin{align}\label{eq: total form in mab}
     \Reg_{\text{FTPL}}(T) \leq \gO\qty(\sqrt{\log T\qty(\omega'(\Delta)+\frac{K}{\Delta_{\min}})(\Reg(T)+C)}) + \gO(\log T).
\end{align}
Therefore, we have
\begin{align*}
     \Reg_{\text{FTPL}}(T) &\leq \gO\qty(\omega(\Delta)\log T + \sqrt{C\omega(\Delta)\log T}) + \gO\qty(\frac{\alpha^3 \sqrt{K}}{(\alpha-1)^2}) \\
    &\hspace{2em}+ \I[\alpha \geq 2]t_0(\alpha,K)  + \I[\alpha \in (1,2)]2,
\end{align*}
where
\begin{align*}
    \omega(\Delta) &= \omega'(\Delta) + \frac{\alpha^2 K}{(\alpha-1)0.31\Delta\min} + \I[\alpha \in (1,2)]\frac{\alpha^2}{\alpha-1} K^{1/\alpha}\\
    &= \gO\qty(\frac{8\alpha^2}{\alpha-1} \qty(\sum_{i \ne i^*} \Delta_i^{-\frac{1}{\alpha-1}})^{1-\frac{1}{\alpha}} \qty(\sum_{i \ne i^*} \Delta_i^{1-\alpha})^{\frac{1}{\alpha}} + \frac{\alpha^4 K}{(\alpha-1)^2 0.31\Delta_{\min}}).
\end{align*}
Note that $\omega(\Delta) \leq \gO\qty(\frac{\alpha^4 K}{(\alpha-1)^2 \Delta_{\min}})$, whose $\alpha$-dependency is square of \citet{pmlr-v247-ito24a}, which can be seen as a drawback of using surrogate instead of explicit probability.
For the comparison with \citet{pmlr-v247-ito24a}, one can see that the relationship between $\gamma$-Tsallis entropy and Fr\'echet-type perturbation with shape $\alpha$, where $\alpha \approx \frac{1}{1-\gamma}$ correspondence observed.
Therefore, $\omega(\Delta)$ can be seen as the results with $\gO\qty(\frac{1}{(\gamma(1-\gamma))^2\Delta_{\min}})$.

\subsection{Bias term by CGR II}\label{app: CGR bias}
So far, we derived the upper bounds of $\Reg_{\text{FTPL}}$, which is the main leading term of the regret.
In this section, we show that $\Reg_{\text{CGR}}$ is at most $\log T$, which does not affect the overall regret.

We start from Lemma 7 of \citet{chen2025geometric} that showed that the expected regret of FTPL with CGR II satisfies
\begin{align*}
    \Reg(T) \leq \sum_{t=1}^T \E\qty[\inp{\hat{\ell}_t}{w_t - e_{i^*}}] + \sum_{t=1}^T \sum_{i=1}^K \E\qty[w_{t,i}\qty(1-\frac{w_{t,i}}{\Pr\qty[\gA_t \middle | \hat{L}_t, i_t=i]})^{G_t}],
\end{align*}
where the second term was denoted by $\Reg_{\textnormal{CGR}}(T)$.
Here, recall the definition of $\gA_t$ in (\ref{def: event for CGR}), which is
\begin{equation*}
    \gA_t = \qty{r_{t,i_t}' = \max_{i  : \sigma_{t,i}\leq \sigma_{t,i_t}}r_{t,i}', \, r_{t,i_t}' \geq \eta_t \hat{\uL}_{t,i_t}}.
\end{equation*}
\begin{lemma}\label{lem: CGR bias bound}
    Algorithm~\ref{alg: FTPL MAB} satisfies that
    \begin{equation*}
        \sum_{t=1}^T \sum_{i=1}^K \E\qty[w_{t,i}\qty(1-\frac{w_{t,i}}{\Pr\qty[\gA_t \middle| \hat{L}_t, i_t=i]})^{G_t}] \leq \log T.
    \end{equation*}
\end{lemma}
\begin{proof}
    Similarly to the proof in \citet{chen2025geometric}, where they consider the Fr\'echet distribution, we denote $\Pr\qty[\gA_t \middle | \hat{L}_t , i_t=i]$ by $\Pr[\gA_{t,i}]$ in this proof.
    Then, it is sufficient to prove 
    \begin{equation*}
        \exp\qty(- \frac{w_{t,i}}{\Pr[\gA_{t,i}]}G_t) \leq \frac{1}{t}.
    \end{equation*}
    Let $\ul_{t} = \eta_t \hat{\uL}_t$.
    By definition of $\gA_{t,i}$, it holds that
    \begin{equation*}
        \Pr[\gA_{t,i}] = \int_{\ul_{t,i}}^\infty f(z) F^{\sigma_{t,i}-1}(z) \dd z = \frac{1-F^{\sigma_{t,i}}(\ul_{t,i})}{\sigma_{t,i}} \leq 1-F(\ul_{t,i}).
    \end{equation*}
    Then, we consider the lower bound of $w_{t,i}$.
    Let $r_{t,-i}^{\max} := \max_{j \ne i} r_{t,j}$ and define an event,
    \begin{equation*}
        \gB_{t,i} := \qty{r_{t,i} \geq \ul_{t,i} + r_{t,-i}^{\max}},
    \end{equation*}
    which is the sufficient condition for $\{i_t = i\}$.
    This is because
    \begin{align*}
        \{i_t = i\} &= \qty{ \argmin_{j \in [K]}\qty{\ul_{t,j}-r_{t,j}}=i} \\
        &=\qty{ \argmax_{j \in [K]}\qty{r_{t,j}-\ul_{t,j}}=i} \\
        &= \qty{\forall j \ne i : r_{t,i}\geq r_{t,j} + \ul_{t,i} - \ul_{t,j}} \\
        &\supseteq \qty{\forall j \ne i : r_{t,i}\geq r_{t,j} + \ul_{t,i}}  = \gB_{t,i}.
    \end{align*}
    Therefore, $w_{t,i} \geq \Pr[\gB_{t,i}|\lambda_t]$.
    Let $\bar{\gP}_{\alpha}^K$ denotes the distribution of $K-1$ block maximum for Pareto distributed random variables, i.e., the distribution of $r_{t,-i}^{\max}$.
    Here, by definition of $\gB_{t,i}$ and definition of $\bar{\gP}_{\alpha,K}$, we have
    \begin{align*}
        \Pr[\gB_{t,i}|\lambda_t] = \E_{r_{t,-i}^{\max} \sim \bar{\gP}_{\alpha}^K}\qty[\Pr(r_{t,i} \geq \ul_{t,i} + r_{t,-i}^{\max}|r_{t,-i}^{\max})]&=\E_{r_{t,-i}^{\max}}[1-F(\ul_{t,i} + r_{t,-i}^{\max})]\\
        &= \E_{r_{t,-i}^{\max}}\qty[\frac{1}{(1+\ul_{t,i} + r_{t,-i}^{\max})^\alpha}] \numberthis{\label{eq: frechet type}} \\
        &\geq \E_{r_{t,-i}^{\max}}\qty[\frac{1}{(1+\ul_{t,i})^\alpha (1+r_{t,-i}^{\max})^\alpha}] \\
        &= (1-F(\ul_{t,i}) \E_{r_{t,-i}^{\max}}[1-F(r_{t,-i}^{\max})].
    \end{align*}
    Here,
    \begin{equation*}
        \E_{r_{t,-i}^{\max}}[1-F(r_{t,-i}^{\max})] = \Pr[r_{t,i} \text{ is the maximum among }\{r_{t,j}\}_{j\in[K]}] = \frac{1}{K}.
    \end{equation*}
    Since $r_{t,i}$s are i.i.d.~from the same perturbations, the probability that $r_{t,i}$ is not the maximum over $K$ samples becomes $1-1/K$.
    Therefore,
    \begin{equation*}
        w_{t,i} \geq \frac{1-F(\ul_{t,i})}{K}, 
    \end{equation*}
    which implies that
    \begin{equation*}
        \frac{\Pr[\gA_{t,i}]}{w_{t,i}} \leq K \implies \exp(-\frac{w_{t,i}}{\Pr[\gA_{t,i}]}G_t) \leq \exp(-\frac{G_t}{K}) = \frac{1}{t}.
    \end{equation*}
    \paragraph{When $\alpha \in (1,2)$.}
    In this case, we sometimes set $G_t = 2\log t$ instead of $K\log t$.
    The precise condition is when $i_t=j_t$ and $\gE_{t,\alpha}$ in (\ref{def: event for decomposition for small alpha}) occurs.
    As we mentioned in Section~\ref{sec: FTPL CGR}, the definition of $\gE_{t,\alpha}$ denotes the case when $w_{t,j_t}\geq 1/2$.
    To be precise, $\gE_{t,\alpha}$ means the case when $\sum_{i\ne j_{t}} \frac{1}{(1+\eta_t \hat{\uL}_{t,i})^\alpha} \leq \frac{1}{2}$.
    Since $w_{t,i}\leq \frac{1}{(1+\eta_t\hat{\uL}_{t,i})^{\alpha}}$ for any $i$ and $t$ (see, Appendix~\ref{app: ineq}), this implies that
    \begin{align*}
        \qty{\sum_{i\ne j_t}w_{t,i} \leq \frac{1}{2}} = \qty{w_{t,j_4t} > \frac{1}{2}} \supset \gE_{t,\alpha}.
    \end{align*}
    Since $\Pr[\gA_{t,j_t}] \leq 1$ and $w_{t,j_t}\geq 1/2$, $G_t = 2\log t$ is still valid to obtain the desired result.
\end{proof}
\begin{remark}
    In the current analysis, the only part that depends on the specific form of the Pareto distribution is in (\ref{eq: frechet type}).
    We expect that this argument can be extended to more general Fr\'echet-type distributions. 
    Indeed, the tail function of a Fr\'echet-type distribution can be expressed as $x^{-\alpha}S_F(x)$ for some slowly varying function $S_F$, which means that the tail function can be written as $S_F(x)(x+1)^{-\alpha}$ for the shifted Fr\'echet-type distribution considered in~\citet[Eq.~(7)]{lee2024follow}.
    Therefore, as long as $S_F(x)$ admits a uniform lower bound by a positive constant, incorporating this constant into the choice of $G_t$ will suffice to obtain the same results.
\end{remark}

\subsection{Proof of Lemma~\ref{lem: improved q and w}}
\begin{proof}
    By definition of $w_{t,i}$, it holds that for $i \ne j_t$
    \begin{align*}
       w_{t,i}&= \int_{0}^\infty f(z+\eta_t\hat{\uL}_{t,i})\prod_{j\ne i} F(z+\eta_t\hat{\uL}_{t,j}) \dd z  \\
       &=  \int_{0}^\infty f(z+\eta_t\hat{\uL}_{t,i})F(z)\prod_{j\ne i, j_t} F(z+\eta_t\hat{\uL}_{t,j}) \dd z \tag{$\because \hat{\uL}_{t,j_t}=0$ on $\gD_{t,\alpha},\, i\ne j_t$} \\
       &\geq \int_{2^{1/\alpha}-1}^\infty f(z+\eta_t\hat{\uL}_{t,i})F(z)\prod_{j\ne i, j_t} F(z+\eta_t\hat{\uL}_{t,j}) \dd z \\
       &\geq \frac{1}{2}\int_{2^{1/\alpha}-1}^\infty f(z+\eta_t\hat{\uL}_{t,i})\prod_{j\ne i, j_t} F(z+\eta_t\hat{\uL}_{t,j}) \dd z \\
       &=\frac{1}{2}\int_{2^{1/\alpha}-1}^\infty f(z+\eta_t\hat{\uL}_{t,i})\prod_{j\ne i, j_t} \qty(1-\frac{1}{(z+\eta_t\hat{\uL}_j+1)^\alpha}) \dd z \\
       &\geq \frac{1}{2} \int_{2^{1/\alpha}-1}^\infty f(z+\eta_t\hat{\uL}_{t,i})\qty(1-\sum_{j \ne i,j_t}\frac{1}{(z+\eta_t\hat{\uL}_j+1)^\alpha}) \dd z \tag{$\prod_i (1-x_i) \geq 1- \sum_{i}x_i$} \\
       &\geq\frac{1}{2} \int_{2^{1/\alpha}-1}^\infty f(z+\eta_t\hat{\uL}_{t,i})\qty(1-\sum_{j \ne i,j_t}\frac{1}{(2^{1/\alpha}+\eta_t\hat{\uL}_j)^\alpha}) \dd z \\
       &\geq \frac{1}{4}  \int_{2^{1/\alpha}-1}^\infty f(z+\eta_t\hat{\uL}_{t,i}) \dd z = \frac{1}{4} \frac{1}{(2^{1/\alpha}+\eta_t \hat{\uL}_{t,j})^\alpha}.
    \end{align*}
    Since $\frac{(x+1)^\alpha}{(x+2^{1/\alpha})^\alpha}$ is increasing with respect to $x \geq 0$ for any $\alpha >1$, we have
    \begin{equation*}
        \frac{(1+\eta_t \hat{\uL}_{t,i})^\alpha}{(2^{1/\alpha}+\eta_t \hat{\uL}_{t,j})^\alpha} \geq \frac{1}{2},
    \end{equation*}
    which concludes the proof for $i\ne j_t$.
    For $i = j_t$, we have
    \begin{align*}
        w_{t,j_t} &= \int_0^\infty \frac{\alpha}{(z+1)^{\alpha+1}} \prod_{i \ne j_t} \qty(1-\frac{1}{(z+\eta_t\hat{\uL}_{t,j}+1)^\alpha}) \dd z \\
        &\geq \int_{2^{1/\alpha}-1}^\infty \frac{\alpha}{(z+1)^{\alpha+1}} \prod_{i \ne j_t} \qty(1-\frac{1}{(z+\eta_t\hat{\uL}_{t,j}+1)^\alpha}) \dd z \\
        &\geq \int_{2^{1/\alpha}-1}^\infty \frac{\alpha}{(z+1)^{\alpha+1}} \qty(1-\sum_{i\ne j_t}\frac{1}{(2^{1/\alpha}+\eta_t\hat{\uL}_{t,j})^\alpha}) \dd z  \\
        &\geq \frac{1}{2}\int_{2^{1/\alpha}-1}^\infty \frac{\alpha}{(z+1)^{\alpha+1}} \dd z = \frac{1}{4},
    \end{align*}
    which concludes the proof.
\end{proof}

\section{Proofs of Lemmas in bandit problems with expert advices}\label{app: lem on contextual}
In this section, we provide the proofs for Lemma~\ref{lem: stability in contextual overall}, which is required to prove BOBW guarantee in the contextual bandit settings.
To prove this lemma, we need the following results.
\begin{lemma}\label{lem: stability in contextual}
    For any $t \in \sN$, Algorithm~\ref{alg: FTPL contextual} with $\alpha >1$ satisfies that for any $i \in [K]$
    \begin{align*}
        \E\qty[\hat{\ell}_{t,i}(\phi_i(\eta_t \hat{L}_t) - \phi_i(\eta_t(\hat{L}_t+ \hat{\ell}_{t})) \middle | \hat{L}_t] \leq \sum_{a=1}^N \E\qty[\frac{2e\alpha}{\beta_t}p_{t,i}^{1/\alpha} \frac{w_{t,i}\pi_{t,i,a}}{P_{t,a}} \middle | \hat{L}_t] .
    \end{align*}
\end{lemma}
Note that $\hat{\ell}_{t,i} \ne 0$ is possible even when $i_t \ne i$ in contextual setting since $\hat{\ell}_{t,i}$ can be updated by using $\ell_{t,a_t}$ and $\pi_{t,i,a_t}$.
\begin{lemma}\label{lem: jt decomposition contextual}
    For any $t \in \sN$, Algorithm~\ref{alg: FTPL contextual} with $\alpha \geq 2$ satisfies that
    \begin{align*}
        \E\qty[\hat{\ell}_{t,j_t}(\phi_{j_t}(\eta_t \hat{L}_t) - \phi_{j_t}(\eta_t(\hat{L}_t+ \hat{\ell}_{t})) \middle | \hat{L}_t] \leq \sum_{a=1}^N \E\qty[\sum_{i\ne j_t} \gO\qty(\frac{\alpha}{\beta_t}p_{t,i}^{1/\alpha}) \frac{w_{t,j_t}\pi_{t,j_t,a}}{P_{t,a}} \middle | \hat{L}_t] + g_t(\alpha; \nu),
    \end{align*}
    where $g_t(\alpha)$ is a function such that $\sum_t g_t(\alpha; \nu)= \gO(\alpha^2/\nu)$.
\end{lemma}
\begin{proof}
\textbf{of Lemma~\ref{lem: stability in contextual overall} }
    From Lemmas~\ref{lem: stability in contextual} and~\ref{lem: jt decomposition contextual}, we obtain that
\begin{align*}
     \E&\qty[\inp{\hat{\ell}_t}{\phi(\eta_t\hat{L}_t) - \phi(\eta_t \hat{L}_{t+1})} \middle | \hat{L}_t] \\
     &\leq \sum_{i =1}^K \E\qty[\hat{\ell}_{t,i}(\phi_i(\eta_t\hat{L}_t) - \phi_i(\eta_t \hat{L}_{t}+\hat{\ell}_{t,i} e_i))\middle| \hat{L}_t] \\
    &\leq \sum_{a=1}^N \sum_{i\ne j_t} \E\qty[\frac{2e\alpha}{\beta_t} p_{t,i}^{1/\alpha}\frac{w_{t,i}\pi_{t,i,a}}{P_{t,a}}\middle | \hat{L}_t] + \sum_{a=1}^N \sum_{i\ne j_t}\E\qty[\gO\qty(\frac{\alpha}{\beta_t}) p_{t,i}^{1/\alpha}\frac{w_{t,j_t}\pi_{t,j_t,a}}{P_{t,a}}\middle | \hat{L}_t] + g_t(\alpha; \nu)\\
    &\leq \sum_{a=1}^N \E\qty[ \gO\qty(\frac{\alpha}{\beta_t}) \max_{j \ne j_t}  p_{t,j}^{1/\alpha}  \cdot \frac{\sum_{i=1}^K w_{t,i}\pi_{t,i,a}}{P_{t,a}}\middle | \hat{L}_t] + g_t(\alpha; \nu)\\
    &\leq \gO\qty(\frac{\alpha N}{\beta_t}) \max_{i \ne j_t}p_{t,i}^{1/\alpha}  + g_t(\alpha; \nu),
\end{align*}
which concludes the proof.
\end{proof}
\subsection{Proof of Lemma~\ref{lem: stability in contextual}}
\begin{proof}
    While the loss estimators can be updated for all experts $i \in \{j\in [K]: \pi_{t,j,a_t} >0\}$ in the contextual setting, we can apply the intermediate results in Lemma~\ref{lem: stab and pen}.
    The main observation is the increasing property of $\phi_i(\lambda)$ with respect to $\lambda_j$ for $i\ne j$, which is obvious from (\ref{def: ftpl it}) with i.i.d.~perturbations. 
    Therefore, for any $i\in [K]$, we have
    \begin{align*}
        \E\qty[\hat{\ell}_{t,i} \qty(\phi_i(\eta_t\hat{L}_t) - \phi_i (\eta_t \hat{L}_{t+1}))\middle | \hat{L}_t] &\leq \E\qty[\hat{\ell}_{t,i}  \qty(\phi_i(\eta_t\hat{L}_t) - \phi_i (\eta_t(\hat{L}_t + \hat{\ell}_{t,i}e_i)))\middle | \hat{L}_t] \\
        &\leq \E\qty[-\eta_t\hat{\ell}_{t,i}^2 \phi_i'(\eta_t \hat{L}_t)\middle | \hat{L}_t] \tag{by (\ref{eq: stab link mab})}\\ 
        &\leq \E\qty[-\eta_t\ell_{t,a_t}^2 M_t^2\pi_{t,i,a_t}^2\phi_i'(\eta_t \hat{L}_t) \middle | \hat{L}_t] \\
        &\leq  \E\qty[-2\eta_t\pi_{t,i,a_t}^2 \frac{\phi_i'(\eta_t \hat{L}_t)}{P_{t,a_t}^2} \middle | \hat{L}_t] \tag{GR and $\ell_{t} \in [0,1]^N$}\\
        &= \E\qty[-2\eta_t \phi_i'(\eta_t \hat{L}_t)\sum_{a=1}^N \frac{\pi_{t,i,a}^2}{P_{t,a}}\middle | \hat{L}_t] \\
        &\leq \E\qty[-2\eta_t \phi_i'(\eta_t \hat{L}_t)\sum_{a=1}^N \frac{\pi_{t,i,a}}{P_{t,a}}\middle | \hat{L}_t] \\
        &= \E\qty[-2\eta_t \frac{\phi_i'(\eta_t \hat{L}_t)}{w_{t,i}}\sum_{a=1}^N \frac{w_{t,i}\pi_{t,i,a}}{P_{t,a}}\middle | \hat{L}_t] \\
        &\leq \sum_{a=1}^N \E\qty[2e\alpha\eta_t p_{t,i}^{1/\alpha}\frac{w_{t,i}\pi_{t,i,a}}{P_{t,a}}\middle | \hat{L}_t],
    \end{align*}
    where the last inequality follows from the results on the ratio $-\phi_i/\phi_i$ in Appendix~\ref{app: stab and pen proof}.
\end{proof}

\subsection{Proof of Lemma~\ref{lem: jt decomposition contextual}}
Since Algorithm~\ref{alg: FTPL contextual} adopts naive GR, $\beta_t \geq \sqrt{t}$, and $\alpha \geq 2$, we can utilize Lemma 11 in \citet{pmlr-v201-honda23a}, given as follows.
\begin{lemma}[Partial results of Lemma 11 in \citet{pmlr-v201-honda23a}]\label{lem: honda}
    For any $\hat{L}_t \in \sR^K$, $\zeta\in(0,1)$ and $i\in [K]$, if $w_{t,i}\geq w'$ for some fixed constant $w'$, then it holds that
    \begin{align*}
        \qty[\I[\hat{\ell}_{t,i} > \zeta \beta_t] \hat{\ell}_{t,i}\middle | \hat{L}_t] \leq \frac{1}{(1-w')}(1-w')^{\zeta \beta_t}(\zeta \beta_t + 1/w').
    \end{align*}
    Moreover, when $\beta_t \geq a\sqrt{t}$ for some $a >0$, it holds that
    \begin{equation*}
        \sum_{t=1}^\infty  \frac{1}{(1-w')}(1-w')^{\zeta \beta_t}(\zeta \beta_t + 1/w') \leq \gO(1/a^2).
    \end{equation*}
\end{lemma}
\begin{proof}
\textbf{of Lemma~\ref{lem: jt decomposition contextual} }
    Similarly to the proof of Lemma~\ref{lem: jt decomposition MAB}, we used the events for $\alpha \geq 2$ defined by
    \begin{equation*}
        \bar{\gE}_{t,\alpha} :=\qty{\sum_{i\ne j_t} \frac{1}{(1+\eta_t \hat{\uL}_{t,i})^\alpha} \leq \xi_\alpha}.
    \end{equation*}
    for some $\xi_\alpha \in (0,1)$ specified later.

    Then, following the same steps in Appendix~\ref{app: Etc case for jt}, we can obtain that
    \begin{align*}
        \I[ \bar{\gE}_{t,\alpha}^c]p_{t,j_t}^{1/\alpha} = \I[\bar{\gE}_{t,\alpha}^c] 1 \leq \frac{2}{\xi_\alpha} \sum_{i\ne j_t} p_{t,i}^{\frac{1}{\alpha}}.
    \end{align*}
    On $\bar{\gE}_{t,\alpha}$, we have $\eta_t \hat{\uL}_{t,i} \geq \xi_\alpha^{-1/\alpha}-1 >0$.
    Therefore, by the same trick in Lemma~\ref{lem: stability in contextual}, where we bound the stability term of $j_t$ in contextual bandits by that in multi-armed bandits, we obtain that
    \begin{multline*}
         \E\qty[\I[\bar{\gE}_{t,\alpha}^c \cup \{\bar{\gE}_{t,\alpha}, \hat{\ell}_{t,j_t} \leq \zeta_\alpha \beta_t \}]\hat{\ell}_{t,j_t}(\phi_{j_t}(\eta_t \hat{L}_t) - \phi_{j_t}(\eta_t(\hat{L}_t+ \hat{\ell}_{t})) \middle | \hat{L}_t] \\ 
         \leq \sum_{a=1}^N \E\qty[\sum_{i\ne j_t} \gO\qty(\frac{\alpha}{\beta_t}p_{t,i}^{1/\alpha}) \frac{w_{t,j_t}\pi_{t,j_t,a}}{P_{t,a}} \middle | \hat{L}_t],
    \end{multline*}
    where $\zeta_\alpha$ also can be tuned as in MABs.
    
    Therefore, it remains to consider the case $\bar{\gE}_{t,\alpha} \cup \{\hat{\ell}_{t,j_t} > \zeta_\alpha \beta_t\}$ for some $\zeta_\alpha < \xi_\alpha^{-1/\alpha}-1$.
    Since Algorithm~\ref{alg: FTPL contextual} utilize simple GR and $\beta_t \geq \sqrt{t}$ for $\alpha \geq 2$ (which is the current interest), we can utilize Lemma~\ref{lem: honda}.
    Here, note that the results in Lemma~\ref{lem: honda} considers the IW estimator for MABs, i.e., $\hat{\ell}_{t,i} = \I[i_t =i]M_t \ell_{t,i}$, while our setting is $\hat{\ell}_{t,i} = M_t\ell_{t,i} \pi_{t,i,a_t}$.
    Therefore, the condition $\hat{\ell}_{t,j_t}\geq \zeta_\alpha \beta_t$ is related to the condition of $M_t \pi_{t,j_t,a_t} \geq \zeta_\alpha \beta_t$.
    
    Specifically, on $\bar{\gE}_{t,\alpha}$, as shown in MABs, we have $w_{t,j_t} \geq 1-\xi_\alpha$ on $\bar{\gE}_{t,\alpha}$, which implies that $P_{t,a_t} \geq \pi_{t,j_t, a_t}(1-\xi_\alpha) \geq \nu(1-\xi_\alpha)$ by Assumption~\ref{asm: on advice}.
    Therefore, $\nu(1-\xi_\alpha)$ plays the same role in $w'$ in Lemma~\ref{lem: honda}.
    Here, note that we do not need to consider the case $\pi_{t,j_t, a_t}=0$ since this case is included in the case of $\hat{\ell}_{t,j_t} \leq \zeta_\alpha \beta_t$.
    \begin{align*}
         \E&\qty[\I[\bar{\gE}_{t,\alpha}, \hat{\ell}_{t,j_t} >\zeta_\alpha \beta_t]\hat{\ell}_{t,j_t}(\phi_{j_t}(\eta_t \hat{L}_t) - \phi_{j_t}(\eta_t(\hat{L}_t+ \hat{\ell}_{t})) \middle | \hat{L}_t] \\
         &\leq \E\qty[\I[\bar{\gE}_{t,\alpha}, \hat{\ell}_{t,j_t} >\zeta_\alpha \beta_t]\hat{\ell}_{t,j_t} \middle | \hat{L}_t] \\
         &\leq \frac{1}{1-(1-\xi_\alpha)\nu} (1-(1-\xi_\alpha)\nu)^{\zeta_\alpha  \beta_t} (\zeta_\alpha \beta_t + 1/(1-(1-\xi_\alpha)\nu)) \tag{by Lemma~\ref{lem: honda}} \\
         &=: g_t(\alpha;\nu). 
    \end{align*}
    Then, Lemma~\ref{lem: honda} shows that $\sum_t g_t(\alpha) \leq \gO(1)$.
    More precisely, $g_{t}(\alpha;\nu)$ is the upper bounds of
    \begin{align}\label{eq: gt order}
         P_{t,a_t} \sum_{m=\lfloor\frac{\zeta_\alpha \beta_t}{\pi_{t,j_t, a_t}}\rfloor+1}^{\infty} m (1- P_{t,a_t})^{m-1} &\leq (1-P_{t,a_t})^{\floor{\frac{\zeta_\alpha \beta_t}{\pi_{t,j_t, a_t}}}}\qty(\floor{\frac{\zeta_\alpha \beta_t}{\pi_{t,j_t, a_t}}} + \frac{1}{P_{t,a_t}}),
    \end{align}
    where $P_{t,a_t} = \sum_{i=1}^Kw_{t,i} \pi_{t,i, a_t} \geq w_{t,j_t} \pi_{t,j_t,a_t}$.
    The introduction of $\nu$ is to obtain a general upper bound by removing the dependency of $\pi_{t,j_t,a_t}$ in $g_t(\alpha)$, which we cannot control.
    
    Finally, we show the order of $\sum_t g_t(\alpha)$ in terms of $\nu$.
    From (\ref{eq: gt order}), we consider the order of
    \begin{align*}
        (1-P_{t,a_t})^{\floor{\frac{\zeta_\alpha \beta_t}{\pi_{t,j_t, a_t}}}}\qty(\floor{\frac{\zeta_\alpha \beta_t}{\pi_{t,j_t, a_t}}} + \frac{1}{P_{t,a_t}}) \leq (1- \pi_{t,j_t,a_t})^{\floor{\frac{\zeta_\alpha \beta_t}{\pi_{t,j_t, a_t}}}}\qty(\floor{\frac{\zeta_\alpha \beta_t}{\pi_{t,j_t, a_t}}} + \pi_{t,j_t,a_t}).
    \end{align*}
    Therefore, it is sufficient to consider the order of
    \begin{equation*}
        \sum_{t=1}^\infty (1-a)^{\frac{b\sqrt{t}}{a}}\qty(\frac{b\sqrt{t}}{a} + a),\, \text{where } a,b \in (0,1).
    \end{equation*}
    It is easy to see that it is decreasing with respect to $a \in (0,1)$.
    Since we have
    \begin{align*}
         (1-a)^{\frac{b\sqrt{t}}{a}}\qty(\frac{b\sqrt{t}}{a} + a) &\leq  e^{-b\sqrt{t}}\qty(\frac{b\sqrt{t}}{a} + a) \\
         &= \frac{b\sqrt{t}e^{-b\sqrt{t}}}{a} + ae^{-b\sqrt{t}},
    \end{align*}
    this implies that
    \begin{align*}
         \sum_{t=1}^\infty (1-a)^{\frac{b\sqrt{t}}{a}}\qty(\frac{b\sqrt{t}}{a} + a) \leq \frac{2e^{-b} (b^2+2b+2)}{ab^2} + \frac{2ae^{-b}(b+1)}{b^2} \leq \gO\qty(\frac{1}{ab^2}+\frac{a}{b^2}).
    \end{align*}
    Therefore, $\sum_{t} g_t(\alpha;\nu) = \gO(\alpha^2/\nu)$ if we choose $\zeta_\alpha\approx 1-2^{-1/\alpha}$ as in the multi-armed bandit.
\end{proof}

\section{Proof of Theorem~\ref{thm: bobw contextual}}\label{app: thm on contextual}
In this section, we show the BOBW guarantee of Algorithm~\ref{alg: FTPL contextual}.
The most of the proofs are essentially the same to that of Theorem~\ref{thm: bobw mab}, where the only difference is related to the change of $z_t$.

\subsection{Adversarial regime}
In this regime, it suffices to show that $h_{\max}z_t$ is at most $NK^{1/\alpha}$ from the second term in (\ref{lem: ito 10}).
By definition of $h_t$ and $z_t$ in (\ref{def: hz in contextual}), we obtain
\begin{align*}
    h_t = \sum_{i\ne j_{t+1}} \frac{\alpha}{\alpha-1}q_{t+1,i}^{1-1/\alpha} &\leq \sum_{i\ne j_{t+1}} \frac{\alpha}{\alpha-1} \frac{1}{\sigma_{t+1,i}^{1-1/\alpha}} \\
    &= \sum_{n=2}^K \frac{\alpha}{\alpha-1} \frac{1}{n^{1-1/\alpha}} \leq \frac{\alpha^2}{\alpha-1}((K+1)^{1/\alpha}-1) \leq \frac{\alpha^2}{\alpha-1} K^{1/\alpha}  
\end{align*}
and
\begin{align*}
    z_t = N\alpha \max_{j \ne j_{t+1} }q_{t+1,j}^{1/\alpha} \leq N \alpha 2^{-1/\alpha}.
\end{align*}
Therefore,
\begin{align*}
    h_{\max}\sum_{t=1}^T z_t \leq \frac{\alpha^3}{\alpha-1} (K/2)^{1/\alpha} N T.
\end{align*}
which concludes the proof for the adversarial regime.

\subsection{Adversarial regime with self-bounding constraint}
Basically, we can directly utilize the techniques used to prove BOBW guarantee for MABs given in Appendix~\ref{app: self-bounding mab}.
Therefore, it suffices to show the upper bounds of $h_t z_t$ on $\gD_{t+1,\alpha}$ since we can just use the $\max_t h_t z_t \leq \alpha^3 K^{1/\alpha}N /(\alpha-1) $ in the case of $\gD_{t+1,\alpha}^c$.

Note that $w_{t,j_t} \geq w_{t,i}$ for any $i \in [K]$.
From (\ref{eq: ht wt mab}), we already obtain that
\begin{equation*}
    \I[\gD_{t+1,\alpha}] h_t \leq \I[\gD_{t+1,\alpha}] \frac{\alpha}{\alpha-1} 16^{1-1/\alpha} \qty(\sum_{i\ne i^*} \frac{1}{\Delta_i^{\alpha-1}})^{\frac{1}{\alpha}}\qty(\sum_{i\ne i^*}\Delta_i w_{t+1,i})^{1-1/\alpha}.
\end{equation*}
For $z_t$, by H\"older's inequality, we have
\begin{align*}
   \I[\gD_{t+1,\alpha}] z_t &=  \I[\gD_{t+1,\alpha}] N \alpha \max_{i\ne j_{t+1}} q_{t+1,i}^{1/\alpha} \\
   &\leq  \I[\gD_{t+1,\alpha}] N\alpha 16^{1/\alpha} \max_{i\ne j_{t+1}} w_{t+1,i}^{1/\alpha} \tag{by (\ref{eq: w q bound in mab})}\\ 
   &\leq \I[\gD_{t+1,\alpha}] N\alpha 16^{1/\alpha} \sum_{i\ne j_{t+1}} w_{t+1,i}^{1/\alpha}\\
   &\leq \I[\gD_{t+1,\alpha}] N\alpha 16^{1/\alpha} \sum_{i\ne i^*} w_{t+1,i}^{1/\alpha} \tag{$w_{t+1,j_{t+1}}\geq w_{t+1,i}, \forall i$} \\
   &= \I[\gD_{t+1,\alpha}] N \alpha16^{1/\alpha} \sum_{i\ne i^*} \frac{1}{\Delta_i^{1/\alpha}}(\Delta_i w_{t+1,i})^{1/\alpha}  \\
   &\leq \I[\gD_{t+1,\alpha}] N \alpha 16^{1/\alpha} \qty(\sum_{i\ne i^*} \frac{1}{\Delta_i^{1/(\alpha-1)}})^{1-\frac{1}{\alpha}}\qty(\sum_{i\ne i^*}\Delta_i w_{t+1,i})^{1/\alpha} .
\end{align*}
Therefore, we obtain
\begin{align*}
    \I[\gD_{t+1,\alpha}] h_t z_t &\leq \I[\gD_{t+1,\alpha}] \omega''(\Delta)\inp{\Delta}{w_{t+1}},
\end{align*}
where
\begin{equation*}
    \omega''(\Delta) =  8N\frac{\alpha^2}{\alpha-1} \qty(\sum_{i \ne i^*} \Delta_i^{-\frac{1}{\alpha-1}})^{1-\frac{1}{\alpha}} \qty(\sum_{i \ne i^*} \Delta_i^{1-\alpha})^{\frac{1}{\alpha}},
\end{equation*}
By following the same steps in Appendix~\ref{app: self-bounding mab} until (\ref{eq: total form in mab}), we obtain
\begin{align*}
    \Reg(T) \leq \gO\qty(\sqrt{\log T\qty(\omega''(\Delta) + \frac{\alpha^3 N K^{1/\alpha}}{(\alpha-1)\Delta_{\min}})(\Reg(T)+C)}) + \gO\qty(\sqrt{\frac{NK^{1/\alpha}\alpha^3}{\alpha-1}} + 1/\nu),
\end{align*}
where the last two terms are the term related to $z_1/\beta_1 + \beta_1 K^{1/\alpha}$ and $g_{t}(\alpha,\nu)$.
Note that additional $z_{\max}/\beta_1$ term appear in Lemma~\ref{lem: spm regret} is $\gO(1)$ in this case, so that we exclude.
This result provides
\begin{align*}
    \Reg(T) \leq \gO\qty(\omega'''(\Delta)\log T + \sqrt{C\omega'''(\Delta)\log T} +\sqrt{\frac{NK^{1/\alpha}\alpha^3}{\alpha-1}}+1/\nu),
\end{align*}
where
\begin{align*}
    \omega'''(\Delta) &= \omega''(\Delta) + \frac{\alpha^3 N K^{1/\alpha}}{(\alpha-1)0.31\Delta\min} \\
    &= \gO\qty(\frac{8N\alpha^2}{\alpha-1} \qty(\sum_{i \ne i^*} \Delta_i^{-\frac{1}{\alpha-1}})^{1-\frac{1}{\alpha}} \qty(\sum_{i \ne i^*} \Delta_i^{1-\alpha})^{\frac{1}{\alpha}} + \frac{\alpha^3N K^{1/\alpha}}{(\alpha-1) 0.31\Delta_{\min}}) \\
    &= \gO\qty(\frac{\alpha^3N K^{1/\alpha}}{(\alpha-1) \Delta_{\min}}).
\end{align*}

\end{document}